\def\year{2020}\relax
\documentclass[letterpaper]{article} 
\usepackage{aaai20}  
\usepackage{times}  
\usepackage{helvet} 
\usepackage{courier}  
\usepackage[hyphens]{url}  
\usepackage{graphicx} 
\usepackage{graphicx}  
\usepackage{amsmath}
\newcommand{\kopt}{k_{opt}}
\usepackage[linesnumbered,ruled,vlined]{algorithm2e}
\usepackage{amsmath, amsxtra, amsfonts, amssymb, amstext, mathtools}
\DeclareMathOperator*{\argmax}{arg\,max}

\newcommand{\ignore}[1]{}

\DeclareMathOperator{\Var}{Var}
\title{AAAI Press Formatting Instructions \\for Authors Using \LaTeX{} --- A Guide }
\author{Written by AAAI Press Staff\textsuperscript{\rm 1}\thanks{Primarily Mike Hamilton of the Live Oak Press, LLC, with help from the AAAI Publications Committee}\\ \Large \textbf{AAAI Style Contributions by
Pater Patel Schneider,} \\ \Large \textbf{Sunil Issar, J. Scott Penberthy, George Ferguson, Hans Guesgen}\\ 
\textsuperscript{\rm 1}Association for the Advancement of Artificial Intelligence\\ 
2275 East Bayshore Road, Suite 160\\
Palo Alto, California 94303\\
publications20@aaai.org 
}

\title{Optimization of Chance-Constrained Submodular Functions}

\author{
Benjamin Doerr,\textsuperscript{\rm 1}
Carola Doerr,\textsuperscript{\rm 2}
Aneta Neumann,\textsuperscript{\rm 3}
Frank Neumann,\textsuperscript{\rm 3}
Andrew M. Sutton,\textsuperscript{\rm 4}\\
\textsuperscript{\rm 1}{Laboratoire d'Informatique (LIX), CNRS, \'Ecole Polytechnique, Institut Polytechnique de Paris, Palaiseau, France}\\
\textsuperscript{\rm 2}{Sorbonne Universit\'e, CNRS, LIP6, Paris, France}\\
\textsuperscript{\rm 3}{Optimisation and Logistics, School of Computer Science, The University of Adelaide, Adelaide, Australia}\\
\textsuperscript{\rm 4}{Department of Computer Science,  University of Minnesota Duluth, Duluth, MN, USA}\\
}

\usepackage{color}

\usepackage{graphicx}
\usepackage{tikz}
\usepackage{todonotes}

\usepackage{amsthm,amsmath}

\newtheorem{lemma}{Lemma}
\newtheorem{theorem}{Theorem}

\newcommand{\carola}[1]{\textcolor{blue}{\textbf{CD:} #1}}

\newcommand{\R}{\mathbb{R}}

\begin{document}

\maketitle

\begin{abstract}
  Submodular optimization plays a key role in many real-world problems. In many real-world scenarios, it is also necessary to handle uncertainty, and potentially disruptive events that violate constraints in stochastic settings need to be avoided. In this paper, we investigate submodular optimization problems with chance constraints. We provide a first analysis on the approximation behavior of popular greedy algorithms for submodular problems with chance constraints. Our results show that these algorithms are highly effective when using surrogate functions that estimate constraint violations based on Chernoff bounds. Furthermore, we investigate the behavior of the algorithms on popular social network problems and show that high quality solutions can still be obtained even if there are strong restrictions imposed by the chance constraint.
\end{abstract}

\sloppy{
\section{Introduction}
Many real-world problems optimization problems involve uncertain components such as the execution length of a job, the fraction of ore in a cartload of rocks, the probability of earthquakes, etc. Safe critical systems or expensive productions must limit the potential violation of constraints imposed by such stochastic components. Constraints that explicitly address the probability of violation are known as chance constraints. Chance-constrained optimization deals with optimizing a given problem under the condition that the probability of a constraint violation does not exceed a given threshold probability $\alpha$. 

In this work we study optimization under chance constraints for submodular functions. Submodular functions model problems where the benefit of adding components diminishes with the addition of elements. They form an important class of optimization problems, and are extensively studied in the literature~\cite{Nemhauser:1978,DBLP:journals/mp/NemhauserWF78,vondrak2010submodularity,DBLP:books/cu/p/0001G14,DBLP:conf/icml/BianB0T17,DBLP:conf/kdd/LeskovecKGFVG07,pmlr-v65-feldman17b,pmlr-v97-harshaw19a}.
However, to our knowledge, submodular functions have not yet been studied in the chance-constrained setting.  

\subsection{Our Results}
Given a set $V$, we analyze the maximization of submodular functions 
$f:2^V \to \mathbb{R}$ subject to linear chance constraints $\Pr[W(S)>B] \leq \alpha$ where $S \subseteq V$ and $W(S)$ is the sum of the random weights of the elements in $S$. Here, $B$ is a deterministic constraint boundary and $\alpha$ is a tolerance level. The objective is thus to find a set that maximizes $f$ subject to the probability that its stochastic weights violate the boundary is at most $\alpha$.

Our focus is on \emph{monotone submodular functions,} which are set functions characterized by the property that the function value cannot decrease by adding more elements. The optimization of the considered classes of submodular problems with deterministic constraints has already been investigated by \cite{DBLP:journals/mp/NemhauserWF78,DBLP:conf/kdd/LeskovecKGFVG07}. The theoretical contribution of this paper extends these results to the chance constrained setting.

Since the computation of the chance constraint is usually not efficiently feasible, we assume it is evaluated by using a surrogate function that provides an upper bound on the constraint violation probability $\Pr[\sum_{s \in S}{W(s)}>B]$. This upper bound ensures that the chance constraint is met if the surrogate provides a value of at most $\alpha$. As surrogates, we use popular deviation inequalities such as Chebyshev's inequality and Chernoff bounds. 

We show that using these surrogate functions, popular greedy algorithms are also applicable in the chance-constrained setting. In particular, we analyze the case of uniformly distributed weights with identical dispersion and show that both inequalities only lead to a loss of a factor of $1-o(1)$ compared to the deterministic setting. 

We complement our theoretical work with an experimental investigation on the influence maximization problem in social networks. This investigation empirically analyzes the behavior of the greedy approaches for various stochastic settings. In particular, it shows the effectiveness of using Chernoff bounds for large inputs if only a small failure rate in terms of $\alpha$ can be tolerated.


\subsection{Related work}
Submodular optimization has been studied for a wide range of different constraint types, see, for example,~\cite{DBLP:books/cu/p/0001G14} and references mentioned therein. Many of the results on monotone submodular functions are based on simple greedy selection strategies that are able to achieve provably the best possible approximation ratio in polynomial time, unless P=NP~\cite{DBLP:journals/mp/NemhauserWF78}. \cite{DBLP:conf/aaai/000100QR19} recently showed that greedy approaches are also successful when dealing with non-monotone submodular functions. Furthermore, Pareto optimization approaches can achieve the same worst-case performance guarantees while performing better than greedy approaches in practice if the user allows for a sufficiently large time budget~\cite{DBLP:conf/ijcai/QianSYT17,DBLP:conf/nips/QianYZ15,DBLP:conf/nips/QianS0TZ17}. \cite{DBLP:conf/aaai/RoostapourN0019} showed that the adaptation of greedy approaches to monotone submodular problems with dynamic constraints might lead arbitrarily bad approximation behavior, whereas a Pareto optimization approach can effectively deal with dynamic changes. Evolutionary algorithms for the chance-constrained knapsack problem, which constitutes a subclass of the chance-constrained submodular problems examined in this paper, have been experimentally investigated by \cite{DBLP:conf/gecco/XieHAN019}.

The paper is structured as follows. Next, we introduce the class of submodular optimization problems and the algorithms that are subject to our investigations. Afterwards, we establish conditions to meet the chance constraints based on tail-bound inequalities. We present our theoretical results for chance-constrained submodular optimization for different classes of weights. Building on these foundations, we present empirical results that illustrate the effect of different settings of uncertainty on the considered greedy algorithms for the influence maximization problem in social networks. Finally, we finish with some concluding remarks.

\section{Chance-Constrained Submodular Functions}

Given a set $V = \{v_1, \ldots, v_n\}$, we consider the optimization of a monotone submodular function $f \colon 2^V \rightarrow \R_{\ge 0}$.
A function is called monotone iff for every $S, T \subseteq V$ with $S \subseteq T$, $f(S) \leq f(T)$ holds. A function $f$ is called submodular iff for every $S, T \subseteq V$ with $S \subseteq T$ and $x \not \in T$ we have
$$f(S \cup \{x\}) - f(S) \geq f(T \cup \{x\}) - f(T).$$

We consider the optimization of such a monotone submodular function $f$ subject to a chance constraint where each element $s \in V$ takes on a random weight $W(s)$. Precisely, we are considering constraints of the type
$$\Pr[W(S) > B] \leq \alpha.$$ 
where $W(S)=\sum_{s\in S}{W(s)}$ is the sum of the random weights of the elements and $B$ is the given constraint bound.  The parameter $\alpha$ quantifies the probability of exceeding the bound $B$ that can be tolerated.

It should be noted that for the uniform distribution, the exact joint distribution can, in principle, be computed  as convolution if the random variables are independent. There is also an exact expression for the Irwin-Hall distribution~\cite{MR1326603} which assumes that all random variables are independent and uniformly distributed within $[0,1]$. However, using these approaches may not be practical when the number of chosen items is large. 

\subsection{Greedy Algorithms}

We consider in this work the performance of greedy algorithms for the optimization of chance constrained submodular functions. Our first greedy algorithm (GA, see Algorithm~\ref{alg:GA}) starts with an empty set and subsequently adds in each iteration an element with the largest marginal gain that does not violate the chance constraint. It ends when no further element can be added. Algorithm~\ref{alg:GA} was already investigated by~\cite{DBLP:journals/mp/NemhauserWF78} in the deterministic setting.
Note that the computation of the probability
$\Pr[W(S) > B]$
can usually not be computed exactly and we make use of a surrogate $\widehat{\Pr}[W(S) > B] \leq \alpha$ on this value (see line~5 of Algorithm~\ref{alg:GA}). Since we use upper bounds for the constraint violation probability, we are guaranteed that the constraint is met whenever our surrogate $\widehat{\Pr}$ is at most $\alpha$.

Our second greedy algorithm is the generalized greedy algorithm (GGA), and is listed in Algorithm~\ref{alg:GGA}. The GGA extends the GA to the case in which the elements have different expected weights. It has previously been used in the deterministic setting~\cite{DBLP:journals/ipl/KhullerMN99,DBLP:conf/kdd/LeskovecKGFVG07}.
The algorithm starts with the empty set. In each iteration, it adds an element whose ratio of the additional gain with respect to the submodular function $f$ and the expected weight increase $E[W(S \cup\{v\})-W(S)]$ of the constraint is maximal while still satisfying the chance constraint. The algorithm terminates if no further element can be added. At this point, it compares this constructed greedy solution with each of the $n$ solutions consisting of a single element, and returns the solution with the maximal $f$-value subject to the surrogate function is at most $\alpha$.
Note that we are using the exact calculation for $\Pr[W(v)> B]$ when considering a single element in line~9.
Lines~9 and~10 of Algorithm~\ref{alg:GGA} are needed in cases where large items of high profit exist, see~\cite{DBLP:journals/ipl/KhullerMN99,DBLP:conf/kdd/LeskovecKGFVG07} for more details.

\subsection{Concentration Bounds}
\label{sec:concentration}

We work with two different surrogates, which are concentration bounds of Chernoff and Chebyshev type. Such bounds are frequently used in the analysis of randomized algorithms~\cite{Motwani1995}. All bounds are well-known and can be found, e.g., in~\cite{Doerr18bookchapter}.

\begin{theorem}[Multiplicative Chernoff bound]
\label{thm:chernoff}
  Let $X_1, \ldots, X_n$ be independent random variables taking values in $[0,1]$. Let $X = \sum_{i = 1}^n X_i$. Let $\epsilon \ge 0$. Then 
  \begin{align}
  \Pr[X \ge (1+\epsilon) E[X]] 
  &\le \bigg(\frac{e^\epsilon}{(1+\epsilon)^{1+\epsilon}}\bigg)^{E[X]} \label{chernoff:exp} \\
  \le &\exp\bigg(-\frac{\min\{\epsilon^2,\epsilon\} E[X]}{3}\bigg).\label{eqprobCMUlin2}
  \end{align}
For $\epsilon \le 1$, \eqref{eqprobCMUlin2} simplifies to 
  \begin{equation}
  \Pr[X \ge (1+\epsilon) E[X]] \le \exp\bigg(-\frac{\epsilon^2 E[X]}{3}\bigg).\label{eqprobCMUeasy}
  \end{equation}
\end{theorem}	

For our experimental investigations, we work with equation~\eqref{chernoff:exp}, whereas equation~\eqref{eqprobCMUeasy} is used through our theoretical analysis. Note that equation~\eqref{eqprobCMUeasy} gives the weaker bound. Therefore, our theoretical results showing approximation guarantees also hold when working with equation~\eqref{chernoff:exp}.
Chernoff bounds are very useful when requiring very small values of $\alpha$. For larger values of~$\alpha$, e.g. $\alpha=0.1$, we often get better estimates when working with a variant of Chebyshev's inequality. As we are only interested in the probability of exceeding a given constraint bound, we consider a one-sided Chebyshev inequality (also known as Cantelli's inequality), which estimates the probability of exceeding the expected value taking into account the variance of the considered random variable.

\begin{theorem}[(One-sided) Chebyshev's inequality]
\label{thm:1s-chebyshev}
  Let $X$ be a random variable with expected value $E[X]$ and variance $\Var[X] > 0$. Then, for all $\lambda > 0$, 
  \begin{align}
  &\Pr[X \ge E[X] + \lambda] \leq \frac{\Var[X]}{\Var[X]+\lambda^2}.
  \end{align}
\end{theorem}

\begin{algorithm}[t]
	\SetKwInOut{Input}{input}
    \Input{
    Set of elements $V$, budget constraint $B$, failure probability $\alpha$.}
    $S \leftarrow\emptyset$\;
		$V^\prime \leftarrow V$\;
\Repeat{$V^\prime \leftarrow \emptyset$}{$v^*\leftarrow \argmax_{v\in V^\prime} (f(S \cup\{v\})-f(S))\label{lineGA:KostenNutzen}$\;
    \If{$\widehat{\Pr}[W(S\cup \{v^*\})> B]\leq \alpha$} {$S\leftarrow S\cup \{v^*\}$\;
    $V^\prime \leftarrow V^\prime\setminus \{v^*\}$\;}}
    \Return{$S$}\;
    \caption{Greedy Algorithm (GA)}\label{alg:GA}
    \end{algorithm}

\begin{algorithm}[t]
	\SetKwInOut{Input}{input}
    \Input{
    Set of elements $V$, budget constraint $B$, failure probability $\alpha$.}
    $S \leftarrow\emptyset$\;
		$V^\prime \leftarrow V$\;
\Repeat{$V^\prime \leftarrow \emptyset$}{$v^*\leftarrow \argmax_{v\in V^\prime}\frac{f(S \cup\{v\})-f(S)}{E[W(S\cup \{v\})-W(S)]\label{line:KostenNutzen}}$\;
    \If{$\widehat{\Pr}[W(S\cup \{v^*\})> B]\leq \alpha$} {$S\leftarrow S\cup \{v^*\}$\;
    $V^\prime \leftarrow V^\prime\setminus \{v^*\}$\;}}
    $v^* \leftarrow \argmax_{\{v\in V;\Pr[W(v)> B] \leq \alpha\} }f(v)$\;
    \Return{$\argmax_{Y\in \{S,\{v^*\}\}}f(Y)$}\;
    \caption{Generalized Greedy Algorithm (GGA)}\label{alg:GGA}
    \end{algorithm}

\section{Chance Constraint Conditions}
We now establish conditions to meet the chance constraint. We start by considering the Chernoff bound given in equation~\eqref{eqprobCMUeasy}.

\begin{lemma}
\label{lem:B-Exp}
Let $W(s) \in [a(s) - \delta, a(s) + \delta]$ be independently chosen uniformly at random. If 
$$(B - E[W(X)]) \geq  \sqrt{3\delta k \ln(1/\alpha)},$$
where $k=|X|$, then $\Pr[W(X)>B] \leq \alpha$.
\end{lemma}

\begin{proof}
Every item has an uncertainty of $\delta$. 
Instead of considering $W(s) \in [a(s)-\delta, a(s)+\delta]$ chosen uniformly at random, we can consider $W'(s) \in [0,2\delta]$ chosen uniformly at random and have $W(s) = a(s) -\delta + W'(s)$. 
For a selection $X$ with $|X|=k$ elements, we can therefore write $W(X) = E[W(X)] - \delta k + \sum_{x \in X} W'(X)$. 

We have $E[W'(X)] = \delta k$.
We consider the probability for exceeding this expected value by $\epsilon \delta k$. We set $\epsilon = (B - E[W(X)])/(\delta k)$ which implies $\epsilon \delta k +E[W(X)]=B$.

We investigate
$$
\Pr [W(X) > B] = \Pr [W'(X) > \epsilon \delta k +k\delta].
$$

Note that if $\epsilon = (B - E[W(X)])/(\delta k)> 1$ then $\Pr[W(X)>B]=0$ as all weights being maximal within their range would not exceed the bound $B$. For $\epsilon\leq 1$, we get
\begin{align*}
\Pr[W(X) > B] &= \Pr[W'(X) > \epsilon \delta k +k\delta] \\
&\leq \exp\bigg(-\frac{\epsilon^2 k\delta}{3}\bigg)
\end{align*}

using equation~\eqref{eqprobCMUeasy}.
In order to meet the chance constraint, we require
\begin{eqnarray*}
& & \exp\bigg(-\frac{\epsilon^2 k\delta}{3}\bigg) \leq \alpha\\
& \Longleftrightarrow & -\frac{\epsilon^2 k\delta}{3} \leq \ln(\alpha)\\
& \Longleftrightarrow & \epsilon^2 k\delta \geq 3\ln(1/\alpha)\\
& \Longleftrightarrow & \epsilon^2 \geq (3\ln(1/\alpha))/ (k\delta).
\end{eqnarray*}

This implies that 
$ \epsilon \geq \sqrt{(3\ln(1/\alpha))/ (k\delta)}  
$
meets the chance constraint condition according to the considered Chernoff bound.
Setting $\epsilon = (B - E[W(X)])/(\delta k)$ leads to 
\begin{eqnarray*}
&  &  (B - E[W(X)])/(\delta k) \geq \sqrt{(3\ln(1/\alpha))/ (k\delta)} \\
& \Longleftrightarrow &  (B - E[W(X)]) \geq  \sqrt{3\delta k \ln(1/\alpha)},   
\end{eqnarray*}
which completes the proof.
\end{proof}
 Based on Chebyshev's inequality, we can obtain the following condition for meeting the chance constraint.

\begin{lemma}
\label{lem:C-Exp}
Let $X$ be a solution with expected weight $E[W(X)]$ and variance $\Var[W(X)]$. If
$$
B- E[W(X)] \geq \sqrt{\frac{(1 - \alpha) \Var[W(X)]}{\alpha}}
$$
 then $\Pr[W(X)>B] \leq \alpha$.
\end{lemma}

\begin{proof}


We have
\begin{align*}
  & \frac{\Var[W(X)]}{\Var[W(X)] + (B- E[W(X)])^2} \leq \alpha \\
  &\Longleftrightarrow  \Var[W(X)]\leq \alpha  (\Var[W(X)] + (B- E[W(X)])^2)\\
  &\Longleftrightarrow  (1- \alpha)  \Var[W(X)]\leq \alpha (B- E[W(X)])^2\\
  & \Longleftrightarrow (B- E[W(X)])^2 \geq \frac{(1-\alpha) \Var[W(X)]}{\alpha}
\end{align*}

This together with Lemma~\ref{lem:C-Exp} implies that the chance constraint is met if
$$B- E[W(X)] \geq \sqrt{\frac{(1-\alpha) \Var[W(X)]}{\alpha}}
$$
holds.
\end{proof}

\section{Uniform IID Weights}
We first study the case that all items have iid weights $W(s) \in [a-\delta,a+\delta]$ ($\delta \le a$). For this case we prove that the greedy algorithm with the Chernoff bound surrogate achieves a $(1-o(1))(1-1/e)$ approximation of the optimal solution for the deterministic setting when $B = \omega(1)$. 
This extends the same bound for the deterministic setting by \cite{DBLP:journals/mp/NemhauserWF78}  to the chance-constrained case.



\begin{theorem}
Let $a>0$ and $0 \leq \delta <a$. Let $W(s) \in [a-\delta, a+\delta]$ be chosen uniformly at random for all $s$. Let $\epsilon(k) = \frac{\sqrt{3\delta k \ln(1/\alpha)}}{a}$ and $k^*$ be the largest integer such that $k+ \epsilon(k) \leq \kopt :=\lfloor B/a \rfloor$.

Then the first $k^*$ items chosen by the greedy algorithm satisfy the chance constraint and are a ${1 - (1/e)\exp(\frac{1+\epsilon(k)}{k^*+1+\epsilon(k)})}$-approximation. For $B = \omega(1)$, this is a $(1-o(1))(1-1/e)$-approximation.
\end{theorem}

\begin{proof}
Let $\kopt=\lfloor B/a \rfloor$ be the number of elements that are contained in an optimal solution $OPT$ in the case that the weights are deterministic and attain the value $a$. 

Having produced a solution with $k$ elements following the greedy procedure, we have obtained a solution $X$ where
\[
f(X) \geq (1 - (1- 1/\kopt)^k) \cdot f(OPT)
\]
due to an inductive argument given by \cite{DBLP:journals/mp/NemhauserWF78}.

We now give a lower bound on $k$ using Chernoff bound as a surrogate.
Let $X$ be a set of selected items containing $k=|X|$ elements and $E[X]=\sum_{x \in X} a(x)$ be its expected weight, $\delta$ be the uncertainty common to all items.

Since all items have the same expected weight $a$, we have $E[W(X)] = ak$.
Using Lemma~\ref{lem:B-Exp}, the chance constraint is met if
$(B - ak) \geq \sqrt{3\delta k \ln(1/\alpha)}$.
We have $\kopt = \lfloor B/a \rfloor$ for the number of elements that could be added if the weights were deterministic. So any $k$ with
$
 k+ \frac{\sqrt{3\delta k \ln(1/\alpha)}}{a}   \leq \kopt
$
is feasible when using the Chernoff bound.

Let
\begin{equation}
    k^*= \max \left\{k \,\middle|\,  k+ \frac{\sqrt{3\delta k \ln(1/\alpha)}}{a}   \leq \kopt \right\}. \label{kstar}
\end{equation}
Then
$$
 \kopt<  (k^*+1)+ \frac{\sqrt{3\delta (k^*+1) \ln(1/\alpha)}}{a} =:\beta(k^*).
$$

Let $X^*$ be a solution with $k^*$ elements constructed by the greedy algorithm. Using the well-known estimate $(1+x) \le e^x$, we bound $f(X^*)$ from below by 

\begin{align*}
  (1 - &(1- 1/\kopt)^{k^*}) \cdot f(OPT)\\
 \geq &  \left(1 - \left(1- \frac{1}{\beta(k^*)}\right)^{k^*}\right) \cdot f(OPT)\\
 \geq & \left(1-\exp\left(-\frac{k^*}{\beta(k^*)}\right)\right) \cdot f(OPT)\\
 = & \left(1-\exp\left(-\frac{k^*}{k^* + 1 + \epsilon(k^*+1)}\right)\right) \cdot f(OPT)\\
 = & \left(1-\frac 1e \exp\left(\frac{1 + \epsilon(k^*+1)}{k^* + 1 + \epsilon(k^*+1)}\right)\right) \cdot f(OPT).
\end{align*}
When $k^* = \omega(1)$, the $\exp(\cdot)$ expression is $(1+o(1))$, yielding the asymptotic part of the claim.
\end{proof}

For comparison, we now determine what can be obtained from using a surrogate based on Chebyshev's inequality. This bound is weaker for small values of $\alpha$, but can be better for larger values of $\alpha$ (depending on the other constants involved). 

We observe that $\Var[W(X)] = |X|\cdot \delta^2/3$. Defining $\tilde\epsilon(k) = \frac{\sqrt{(1-\alpha)k\delta^2}}{\sqrt{3\alpha}a}$ and replacing equation~\eqref{kstar} by 
$$    
k^*= \max \left\{k \mid  k+ \tilde\epsilon   \leq \kopt \right\}
$$
our proof above yields the following theorem.

\begin{theorem}
Let $a>0$ and $0 \leq \delta <a$. Let $W(s) \in [a-\delta, a+\delta]$ be chosen uniformly at random for all $s$. Let $\tilde\epsilon(k) = \frac{\sqrt{(1-\alpha)k\delta^2}}{\sqrt{3\alpha}a}$ and $k^*$ be the largest integer such that $k+ \tilde\epsilon(k) \leq \kopt :=\lfloor B/a \rfloor$.

Then the first $k^*$ items chosen by the greedy algorithm satisfy the chance constraint and are a ${1 - (1/e)\exp(\frac{1+\tilde\epsilon(k)}{k^*+1+\tilde\epsilon(k)})}$-approximation. For $B = \omega(1)$, this is a $(1-o(1))(1-1/e)$-approximation.
\end{theorem}

Note that the main difference between the Chernoff bound and Chebyshev's inequality lies in the confidence level of $\alpha$ that needs to be achieved as the equation using Chernoff only increases logarithmically with $1/\alpha$, whereas the one based on Chebyshev's inequality increases with the square root of $1/\alpha$.

We note that, in principle, Chebyshev's inequality does not require that the items are chosen independently. We can use Chebyshev's inequality and the approach above whenever we can compute the variance. 

\section{Uniform Weights with the Same Dispersion}
\ignore{For independently chosen uniform weights $W(s) \in [a(s)-\epsilon,a(s)+\epsilon]$ ($\epsilon \le \min_{v \in V} a(v)$), a generalized greedy algorithm taking into account the ratio of profit gain vs. expect contribution to the chance constraint achieves a $(1-1/\sqrt{e})$ approximation. \carola{auch hier noch Schmutzterm-Faktoren einfuegen} \carola{ggf bekommen wir auch $1-1/e$ hin mit dem Algo von Khuller-IPL bzw Srividenko 2004} \carola{Erklaeren wie cool es ist, dass wir den normalen generalized Greedy Algo nehmen koennen und nur im Nenner die Kosten durch erwartete Kosten ersetzen muessen}}

We now consider the case that the items may have different random weights $W(s) \in [a(s)-\delta, a(s)+\delta]$. However, we still assume the weights are chosen independently and uniformly at random. We also assume that the uncertainty $\delta$ is the same for all items. 

Let $a_{\max} = \max_{s\in V} a(s)$. We assume that $\frac{a_{\max}+ \delta - B}{2\delta} \leq \alpha$
holds. This means that every single item fulfills the chance constraint. Note that items that would not fulfill this condition could be filtered out in a preprocessing step as they can not be part of any feasible solution. Furthermore, we assume that $B=\omega(1)$ grows with the input size. The following theorem extends the results of Theorem $3$ by \cite{DBLP:conf/kdd/LeskovecKGFVG07} to the chance-constrained setting.

\begin{theorem}\label{thm:idchernoff}
For all $s \in V$ let $W(s) \in [a(s)-\delta, a(s)+\delta]$ with $a(s)>0$ and $0 \leq \delta < \min a(s)$. 
If $\frac{a_{\max}+ \delta - B}{2\delta} \leq \alpha$ and $B = \omega(1)$,
then the solution obtained by the Generalized Greedy algorithm GGA using Chernoff bound~\eqref{eqprobCMUeasy} as surrogate for the chance constraint is a $(1/2-o(1))(1-1/e)$-approximation.
\end{theorem}

\begin{proof}
Let $S$ be the greedy solution constructed by the GGA (Algorithm~\ref{alg:GGA}) and let $v^*$ be the element with the largest function value. Let $T$ be the solution produced by the generalized greedy algorithm in the deterministic setting.
\ignore{The solution produced by the adaptive greedy algorithm  having added $i$ elements in the deterministic setting has quality 
$$
f(X_i) \geq \left( 1 - \prod_{k=1}^i \left(1 - \frac{c_k}{B}\right) \right) f(OPT)
$$
where $B$ is the constraint bound, OPT is an optimal solution for the deterministic problem, and $c_k$ is the (expected) weight of the item added in step $k$.
}

We know that any solution $X$ with 
$E[W(X)] \leq  B - \sqrt{3\delta k \ln(1/\alpha)}$ 
 is feasible. Furthermore, every single item is feasible as 
 $\frac{a_{\max}+ \delta - B}{2\delta} \leq \alpha.$
Let  $\hat{B} = B - \sqrt{3\delta k \ln(1/\alpha)}$. Let $S'=X_L \subseteq T$ be the subset of $T$ consisting of the first $L$ elements exactly as chosen when working with the deterministic bound B. Note that we have $f(S) \geq f(S')$ as $f$ is monotone. If $S'=T$, we know $\argmax_{Y\in \{S,\{v^*\}\}}f(Y)$ is a $(1/2) (1-1/e)$-approximation.

Furthermore, let $v^* \in V \setminus S'$ be the the element with the largest function value. Note that $v^*$ meets the constraint bound based on our assumptions. Let $z = X_{L+1} \setminus X_L \in T$ be the first element of $S$ not added when working with the chance constraint.
If $v^* = z$, then we have $E[W(S)] + a(v^*) \geq \hat{B}$.
As $v^*$ is the single element with the largest function value, we have $f(v^*) \geq f(z)$.


Let $x$ be the element added in iteration $i$ to the partial solution $X_{i-1}$ in order to obtain $X_{i} \subseteq S'$. Furthermore, let $OPT$ be an optimal solution for the deterministic setting. Then following \cite{DBLP:conf/kdd/LeskovecKGFVG07} and using the expected cost value instead of the deterministic ones, we have
\begin{equation*}
f(X_i) - f(X_{i-1})
\geq   \frac{a(x)}{ c(OPT)} \cdot (f(OPT) - f(X_{i-1})) \\
\end{equation*}
which by induction gives
$$f(X_i) \geq  \bigg[1- \prod_{k=1}^{i} \left(1- \frac{a(k)}{B}\right) \bigg] \cdot f(OPT).
$$

\ignore{
This results in 
\begin{eqnarray*}
f(S') +f(v^*) & \geq & \left[1- \prod_{k=1}^{L+1} \left(1- \frac{a(k)}{B}\right) \right] \cdot f(OPT)\\
& \geq & \left[1- \left(1 - \frac{1}{L+1}\right)^{L+1} \right]f(OPT)
\end{eqnarray*}

Need to have relation of $OPT$ to $OPT_{\hat{B}}$ and their f and c values.
}


Every element added to $S'$ would have also been chosen when working with the deterministic bound $B$. Furthermore, the single element $v^*$ providing the largest possible profit is also accepted due to our assumption on $B$ and we have $f(v^*) \geq f(z)$.

We have
\begin{eqnarray*}
f(S') +f(z) & \geq & \bigg[1- \prod_{k=1}^{L+1} \left(1- \frac{a(k)}{B}\bigg) \right] \cdot f(OPT),
\end{eqnarray*}
where we take $a(L+1)$ to be $a(z)$.
Furthermore, we have
\begin{align*}
E[W(S')] + a(z) & = a(z) + \sum_{s \in S'} a(s) \geq \hat{B} \\
&\geq  B - \sqrt{3\delta L \ln(1/\alpha)}
\end{align*}
as $z$ is the first element of $S$ not added under the chance constraint. This implies
\begin{align*}
&f(S) + f(v^*) \geq f(S') + f(v^*) \geq f(S') + f(z)\\
& \geq  \left[1 -\prod_{k=1}^{L+1} \left(1- \frac{a(s)}{B}\right) \right]  f(OPT)\\
& \geq  \left[1 -\prod_{k=1}^{L+1} \left(1- \frac{B-\sqrt{3 \delta (L+1) \ln(1/\alpha)}}{(L+1)B}\right) \right]  f(OPT)\\
& \geq  \left[1 - \left(1- \frac{1}{L+1} +\frac{\sqrt{3 \delta \ln(1/\alpha)}}{\sqrt{L+1}\cdot B}\right)^{L+1} \right]  f(OPT)
\end{align*}
Again, we assume that $\alpha$ and $\delta$ are constants and $B=\omega(1)$ is growing with the input size. This implies
$
f(S) + f(v^*) \geq (1-o(1)) (1-1/e)
$
and therefore 
$\max_{Y\in \{S,\{v^*\}\}}f(Y) \geq (1/2-o(1)) (1-1/e)\cdot f(OPT).$
\ignore{

Put some assumption to get (1/2) (1-1/e) approximation!!!!

Assume that $B \geq c n^{\epsilon} \cdot a_{\max}$, $c>0$ a constant, which implies that any subset of at least $c n^{\epsilon}$ elements is feasible for the bound $B$.
}
\end{proof}

Using Chebyshev's inequality with $\Var[W(X)] = |X|\cdot \delta^2/3$ and replacing
$
\sqrt{3 \delta (L+1) \ln(1/\alpha)}
$
by
$
\sqrt{\frac{(1 - \alpha) (L+1) \delta^2}{3\alpha}} 
$
we obtain the following theorem.

\begin{theorem}
In the situation of Theorem~\ref{thm:idchernoff}, if $\frac{a_{\max}+ \delta - B}{2\delta} \leq \alpha$ and $B = \omega(1)$,
then the solution obtained by the Generalized Greedy algorithm using Chebyshev's inequality as surrogate for the chance constraint is a $(1/2-o(1))(1-1/e)$-approximation.
\end{theorem}

\ignore{
\begin{figure}[t]
\centering
\includegraphics[width=0.34\textwidth]{Experiments/digg2.png}

\vspace{-0.35cm}
\caption{Visualization of social Digg's graph structure~\cite{DBLP:conf/aaai/RossiA15}.}
\label{fig:plot_graph}

\end{figure}
}

\begin{figure*}[t!]
\centering
\rotatebox{0}{\hspace{7mm} Budget = $20$}
\rotatebox{0}{\hspace{23mm} Budget = $50$}
\rotatebox{0}{\hspace{20mm} Budget = $100$}
\rotatebox{0}{\hspace{21mm} Budget = $150$}\\
\rotatebox{90}{\hspace{1.78mm} Chebyshev's inequality}
\rotatebox{90}{\rule{35mm}{1pt}}%
\hspace{0.14cm}
\includegraphics[width=0.22\textwidth]{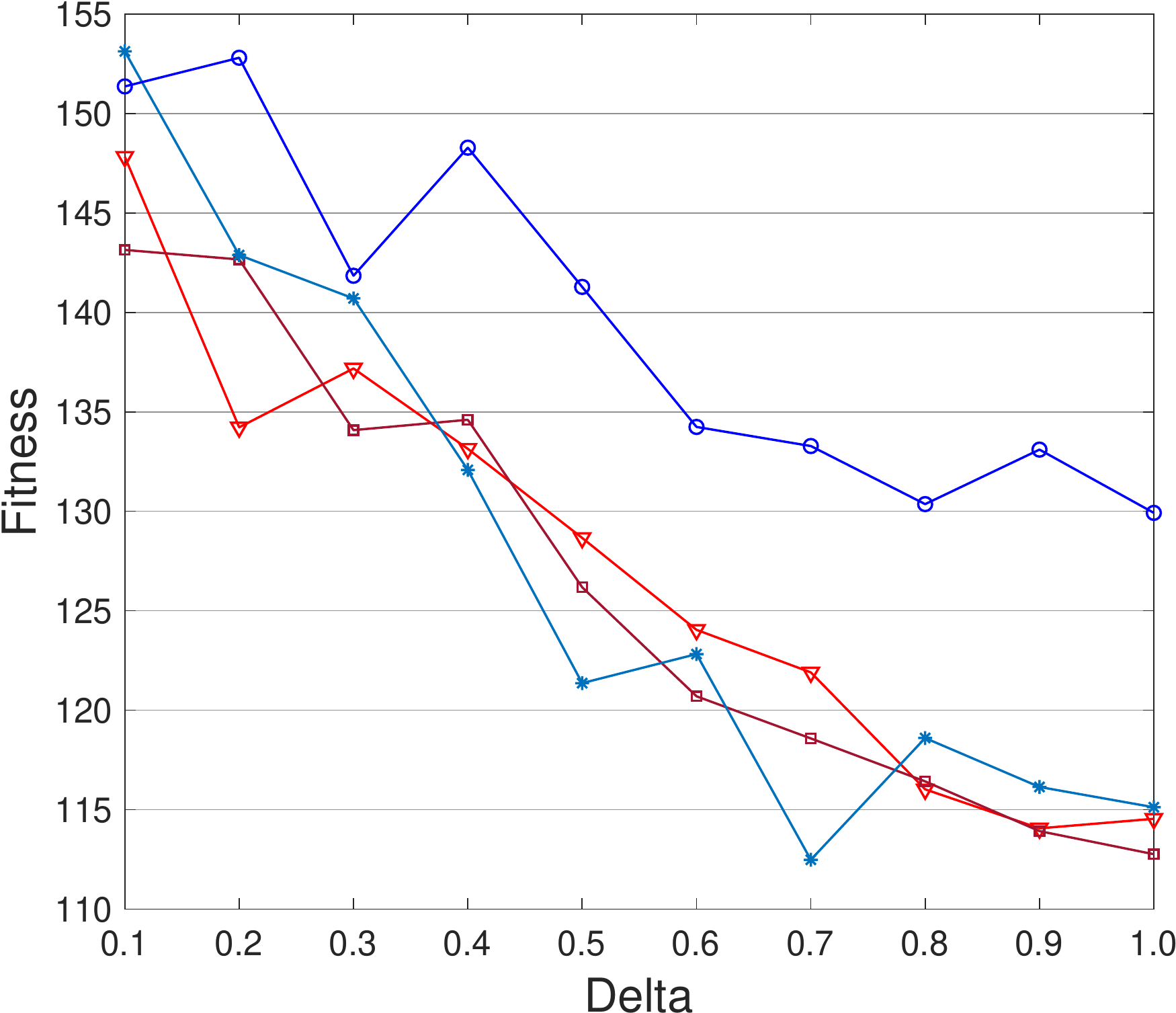}
\includegraphics[width=0.22\textwidth]{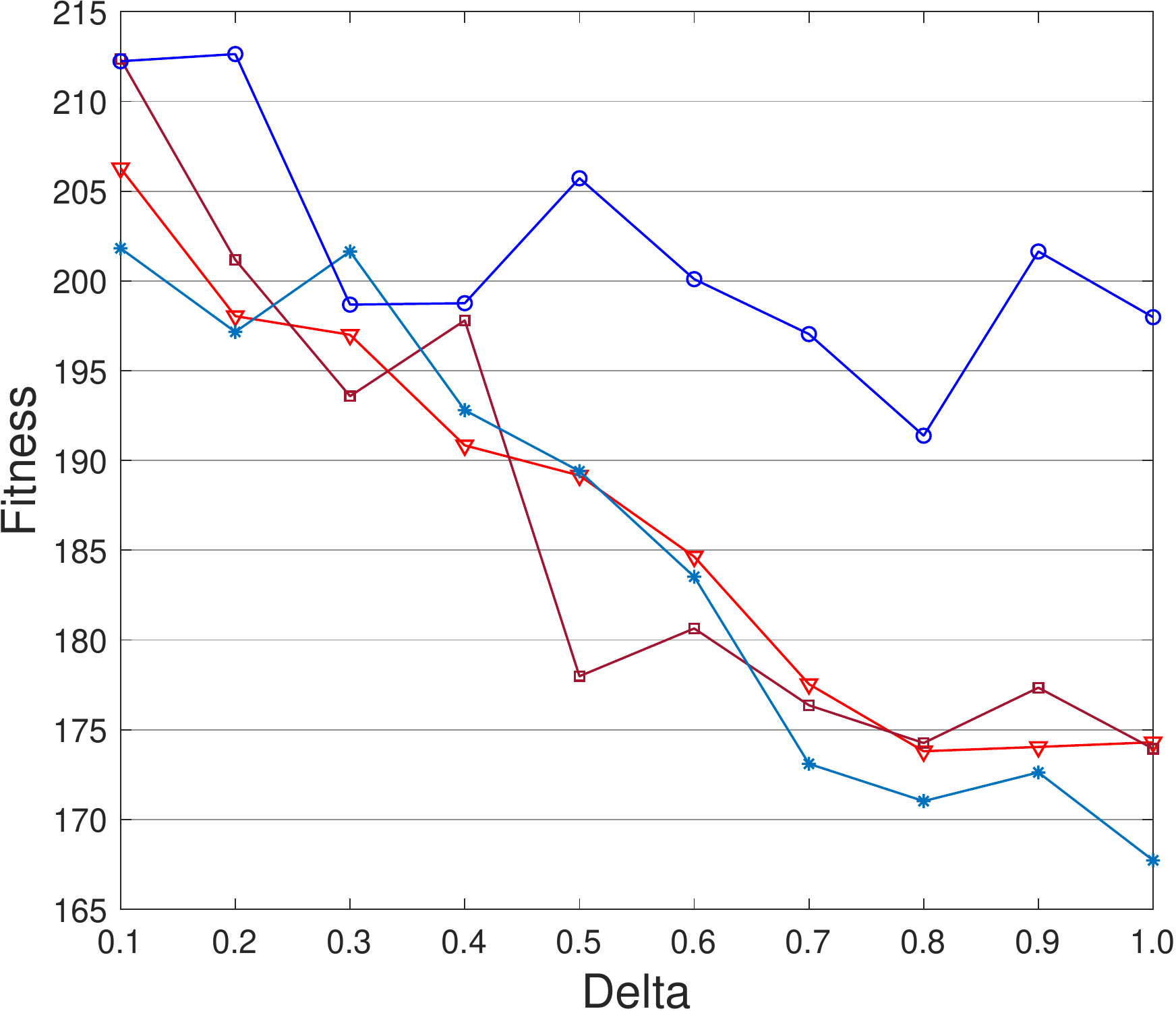}
\includegraphics[width=0.22\textwidth]{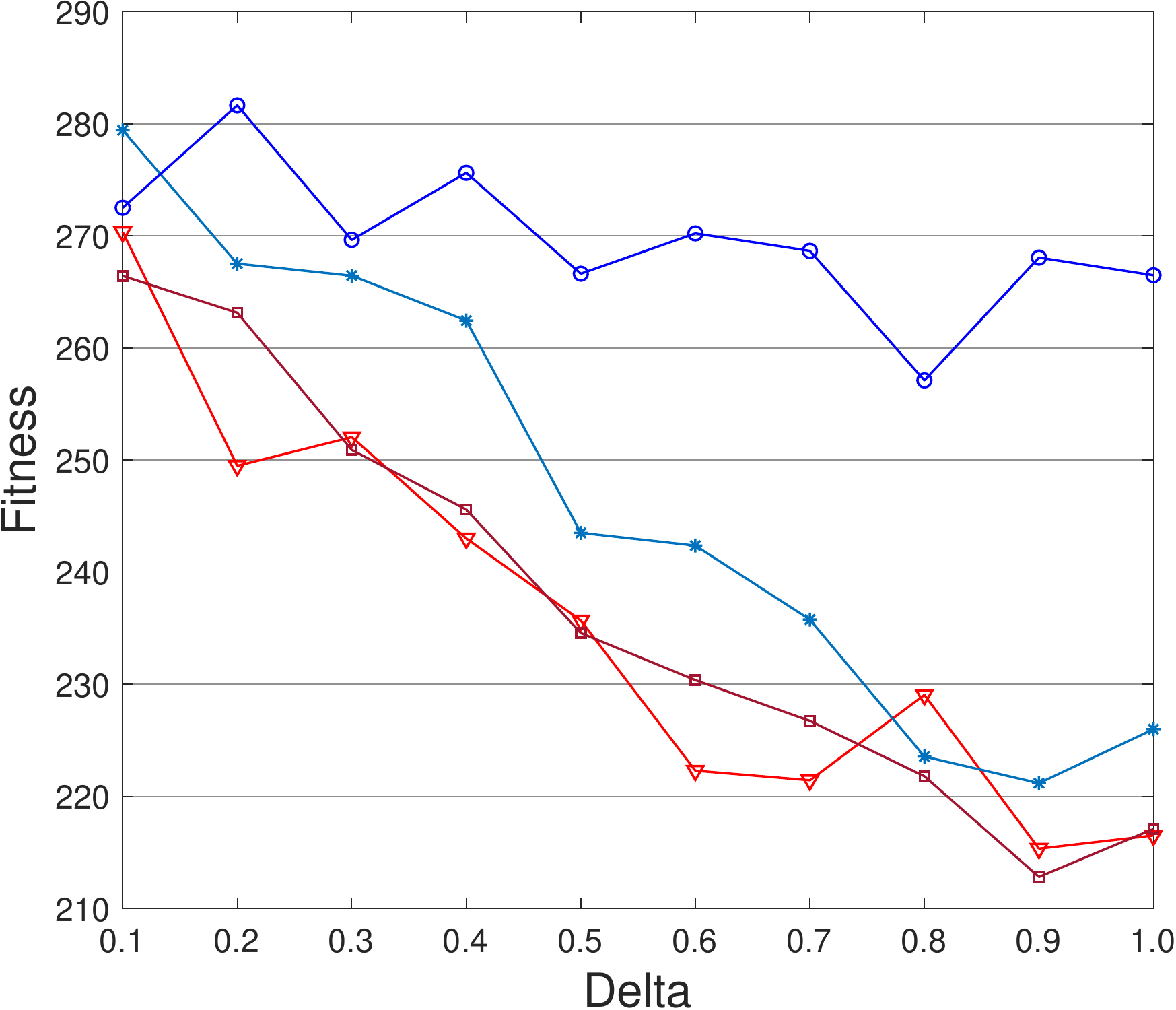}
\includegraphics[width=0.22\textwidth]{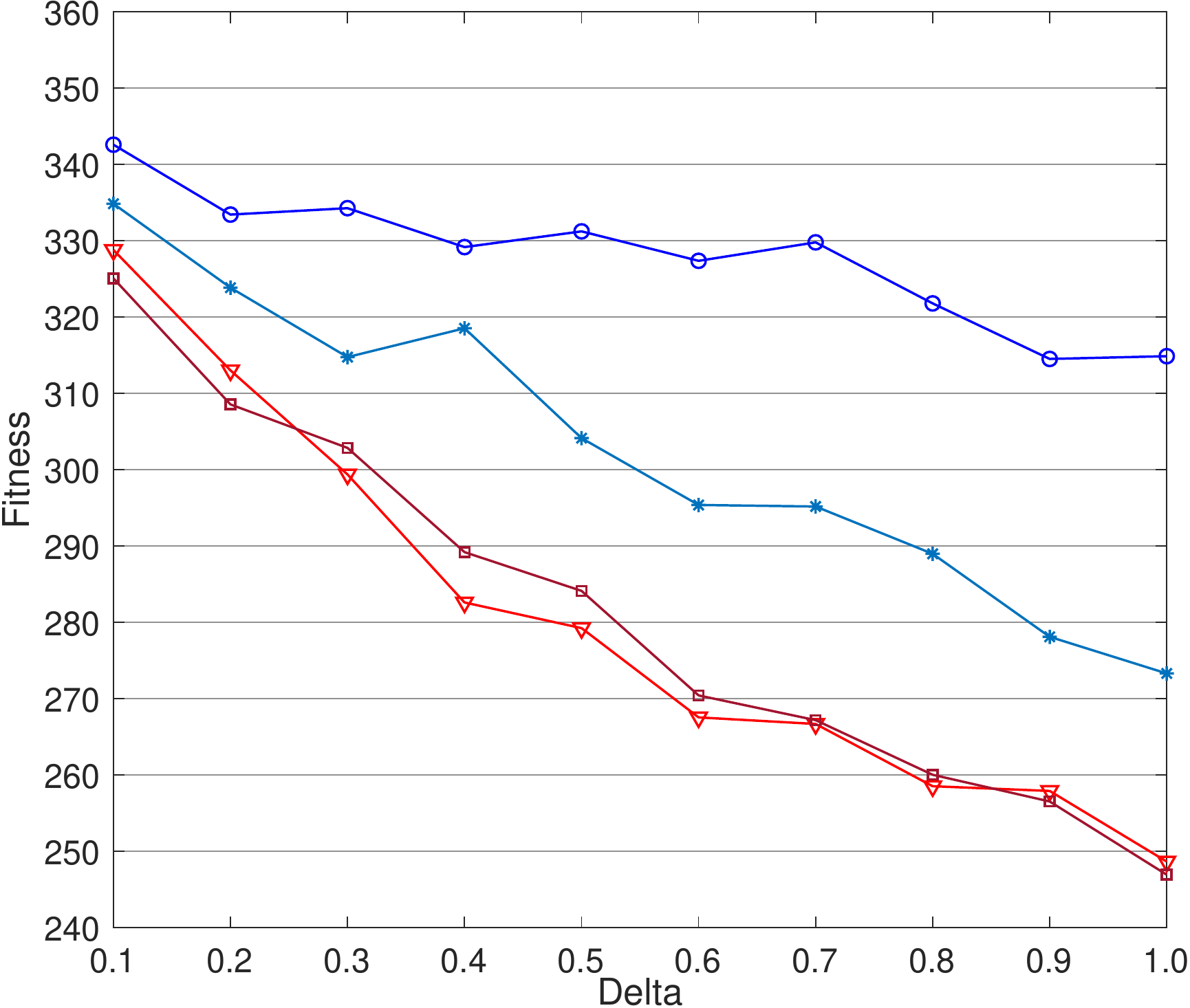}\\
\rotatebox{90}{\hspace{5mm} Chernoff bounds}
\rotatebox{90}{\rule{35mm}{1pt}}%
\hspace{0.12cm}
\includegraphics[ width=0.22\textwidth]{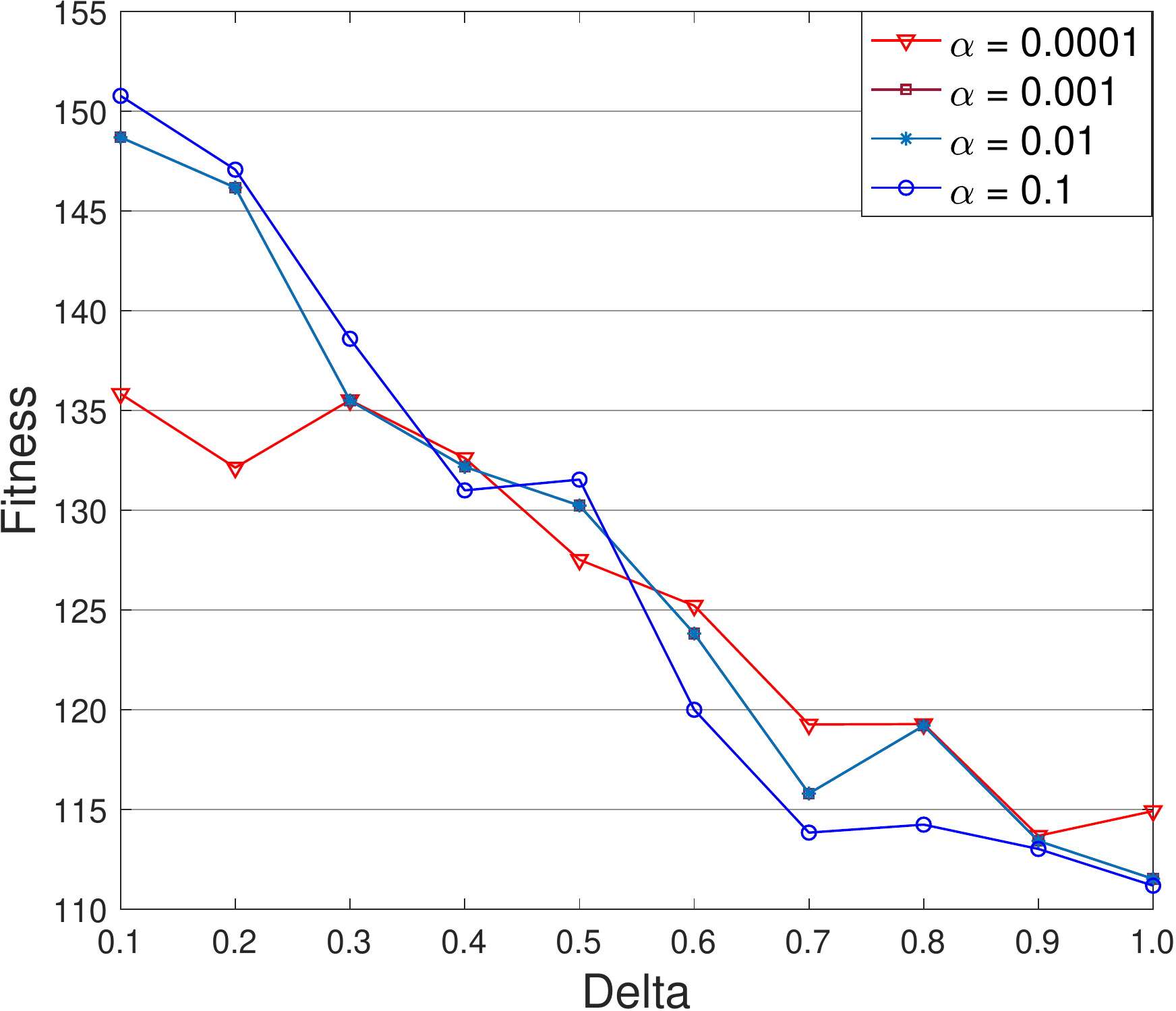}
\includegraphics[ width=0.22\textwidth]{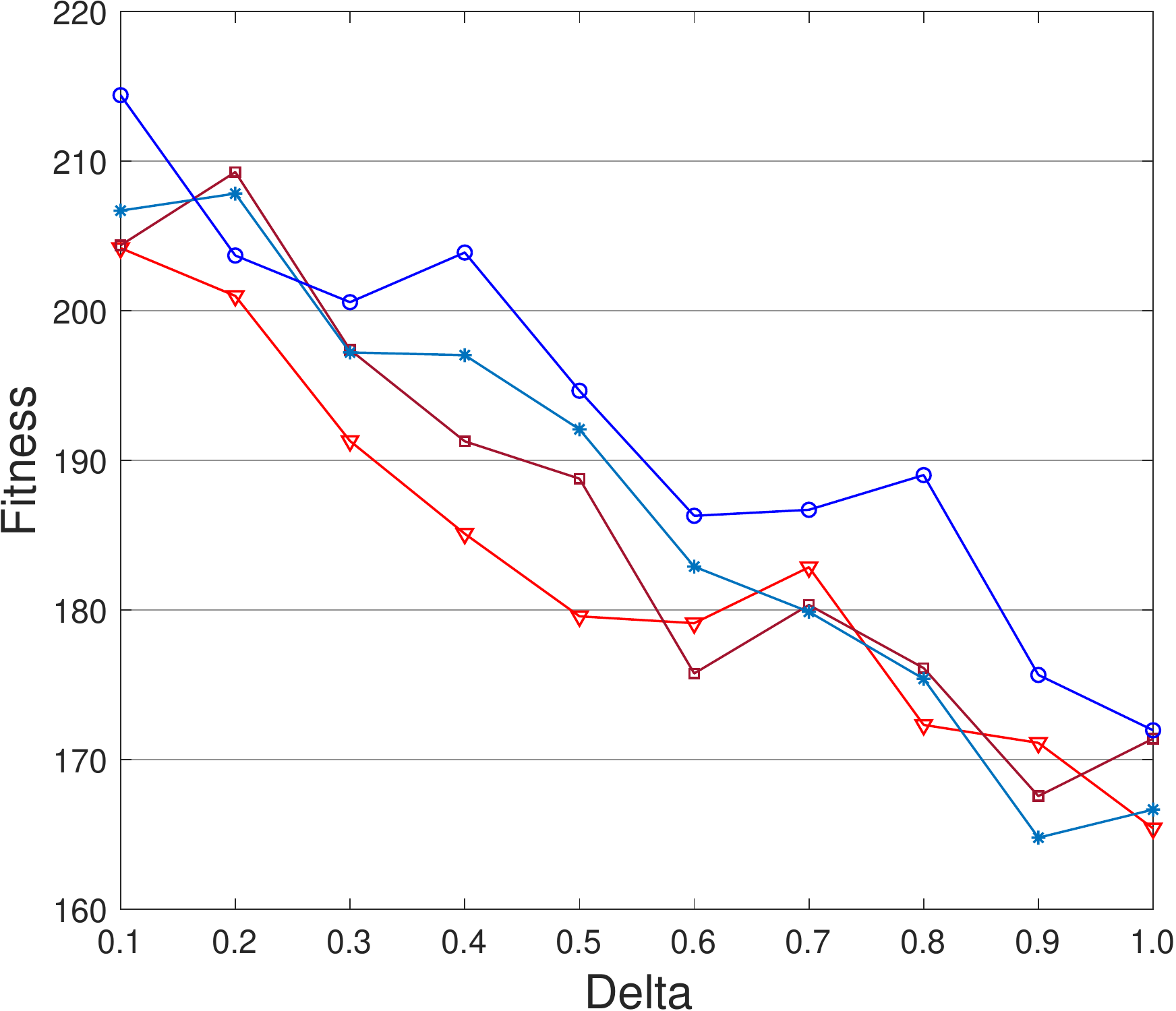}
\includegraphics[ width=0.22\textwidth]{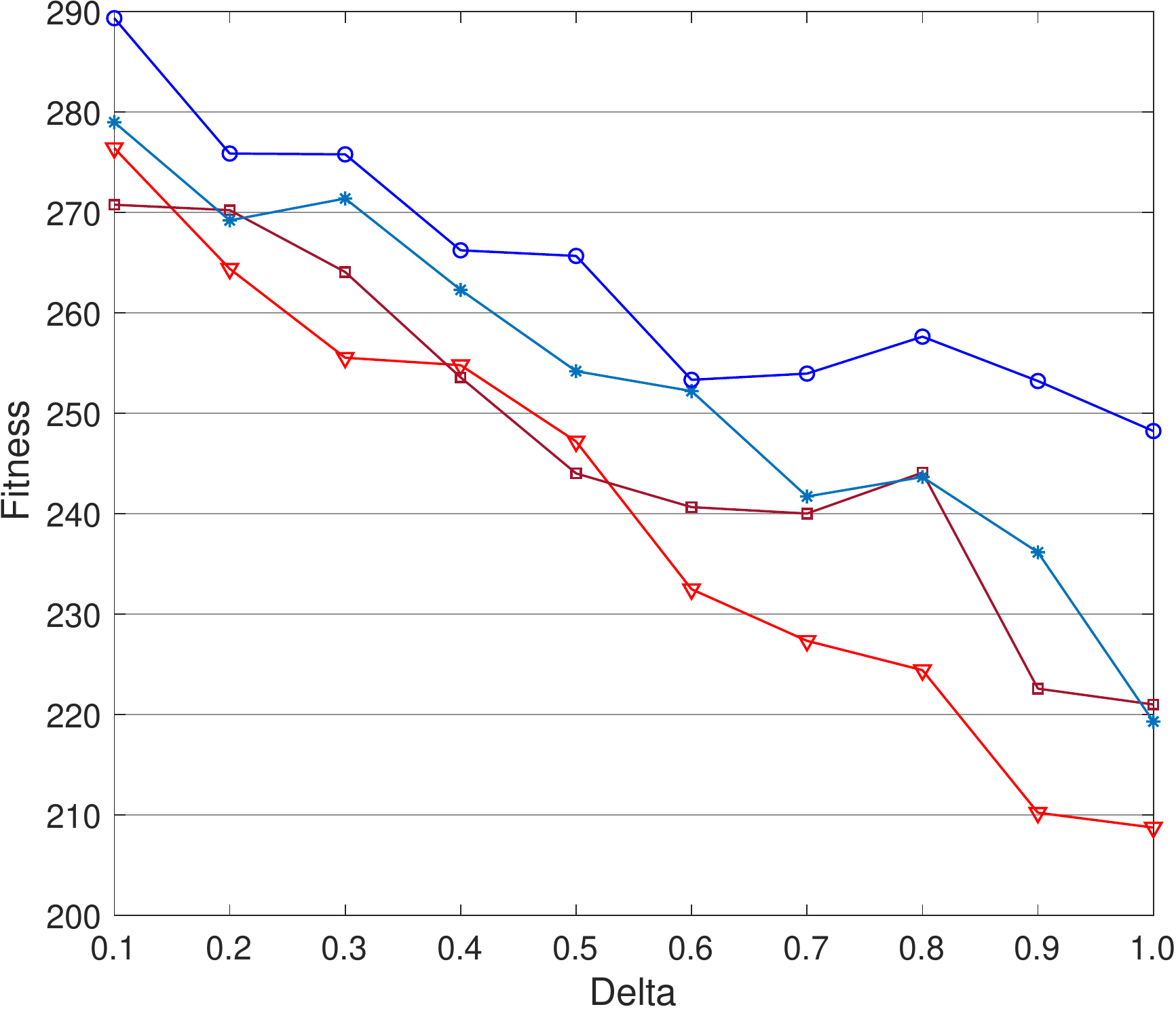}
\includegraphics[ width=0.22\textwidth]{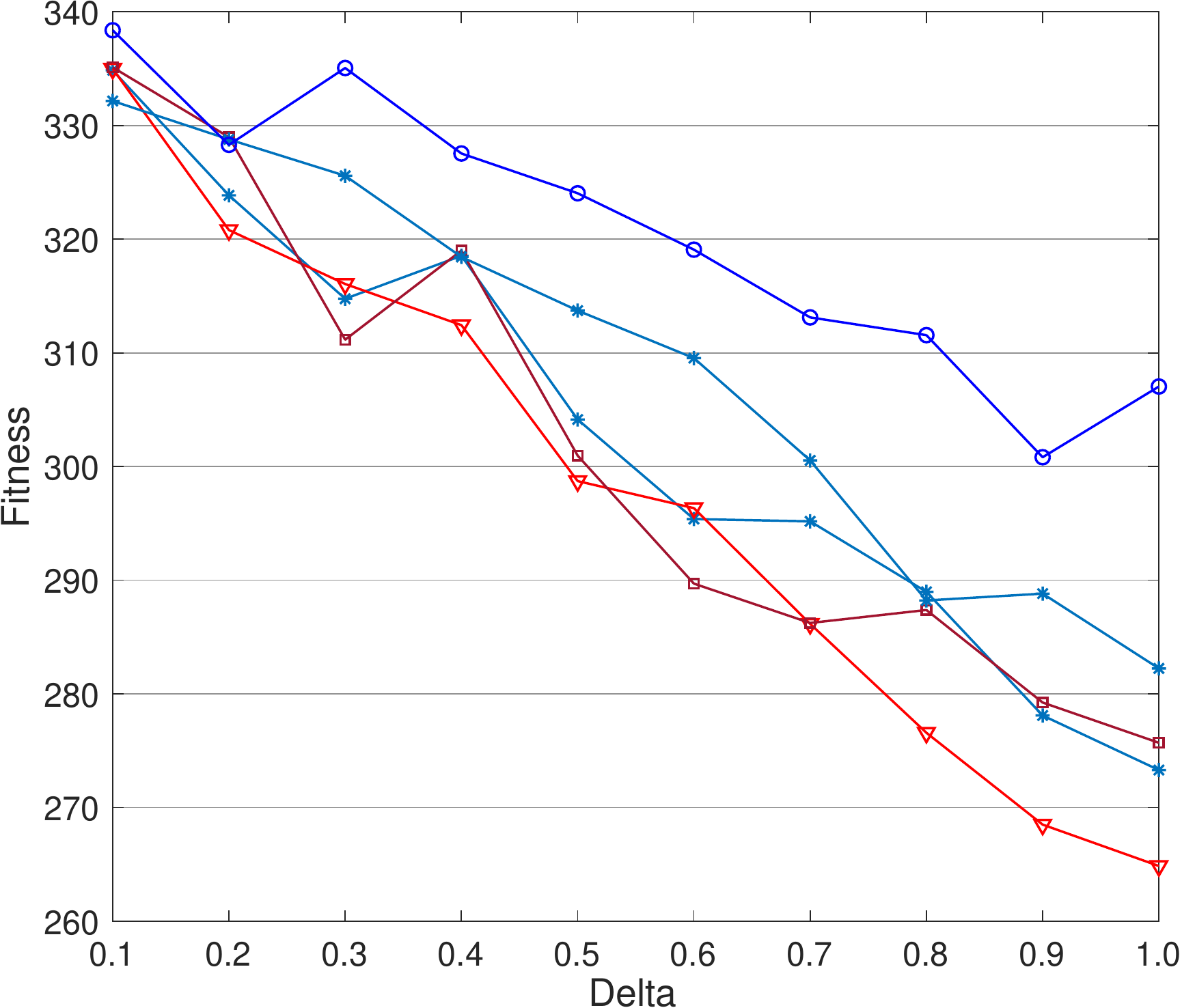}
\vspace{-2mm}
\caption{Function value for budgets $B=20, 50, 100, 150$ (from left to right) using Chebyshev's inequality (top) and Chernoff bound (bottom) for $\alpha$ = 0.1, 0.01, 0.001, 0.0001 with all the expected weights $1$.}
\label{fig:plot}
\end{figure*}


%

\begin{figure*}[t!]
\centering
\includegraphics[width=0.2455\textwidth]{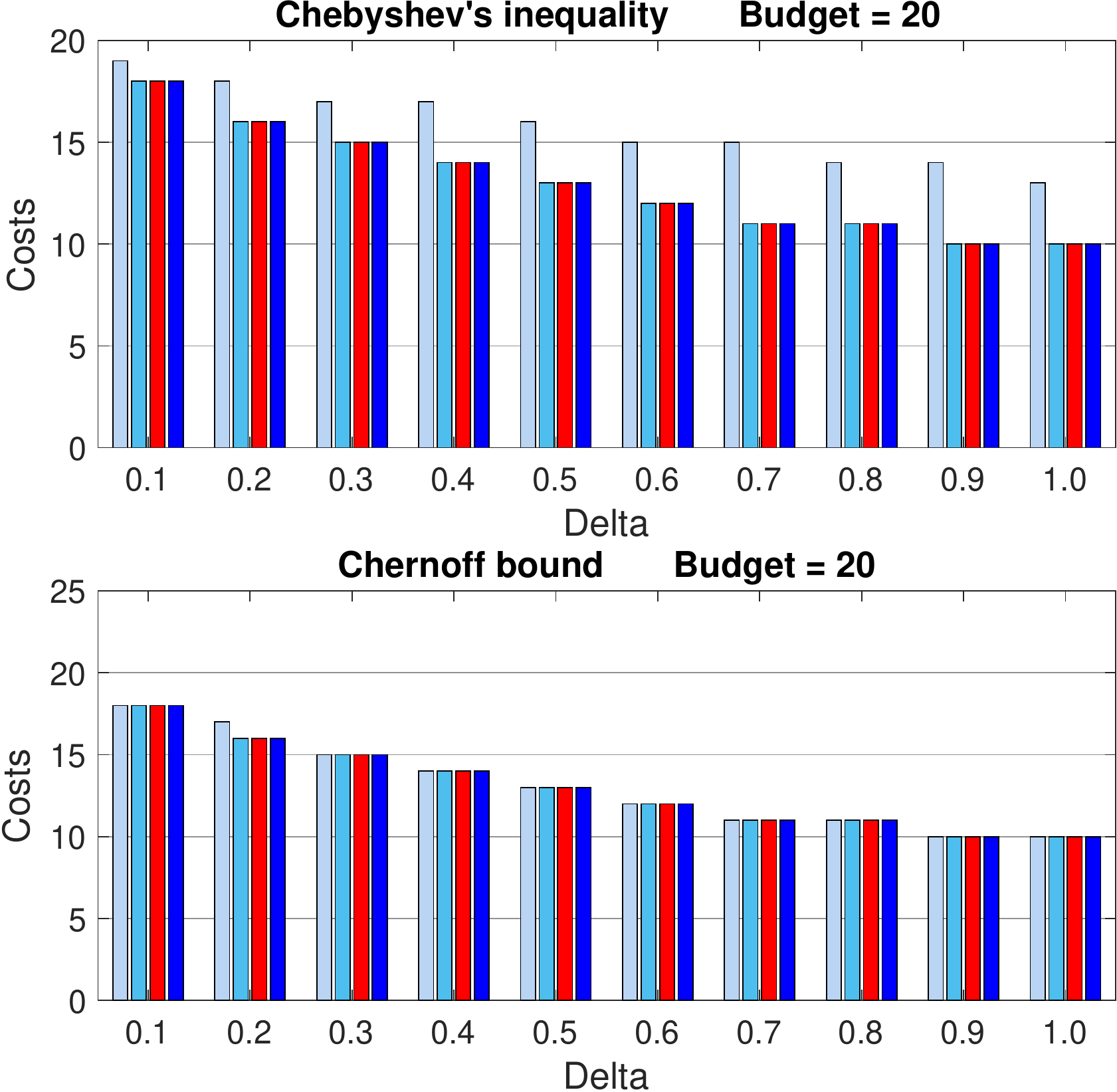}
\includegraphics[width=0.2455\textwidth]{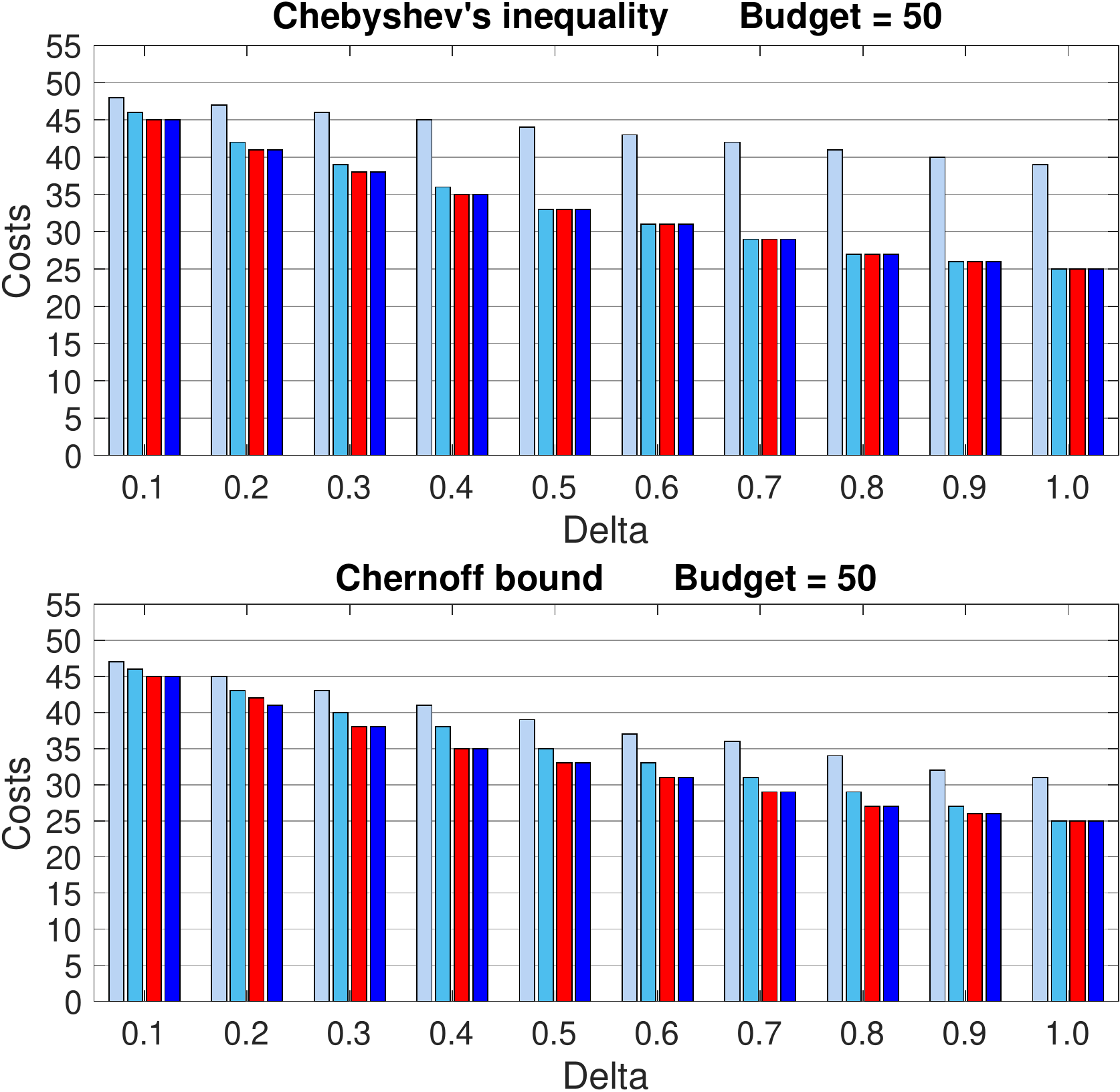}
\includegraphics[width=0.2455\textwidth]{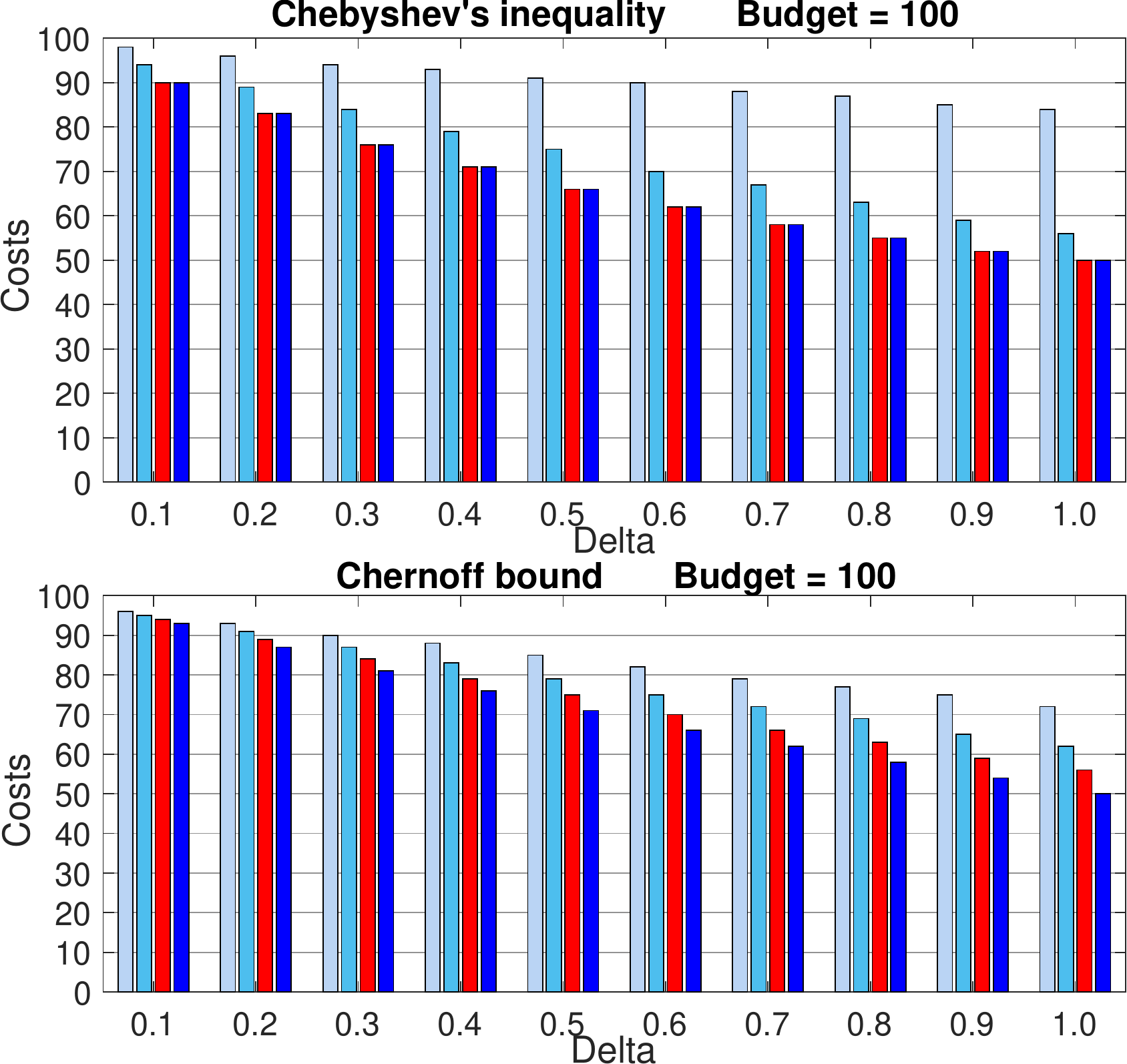}
\includegraphics[width=0.2455\textwidth]{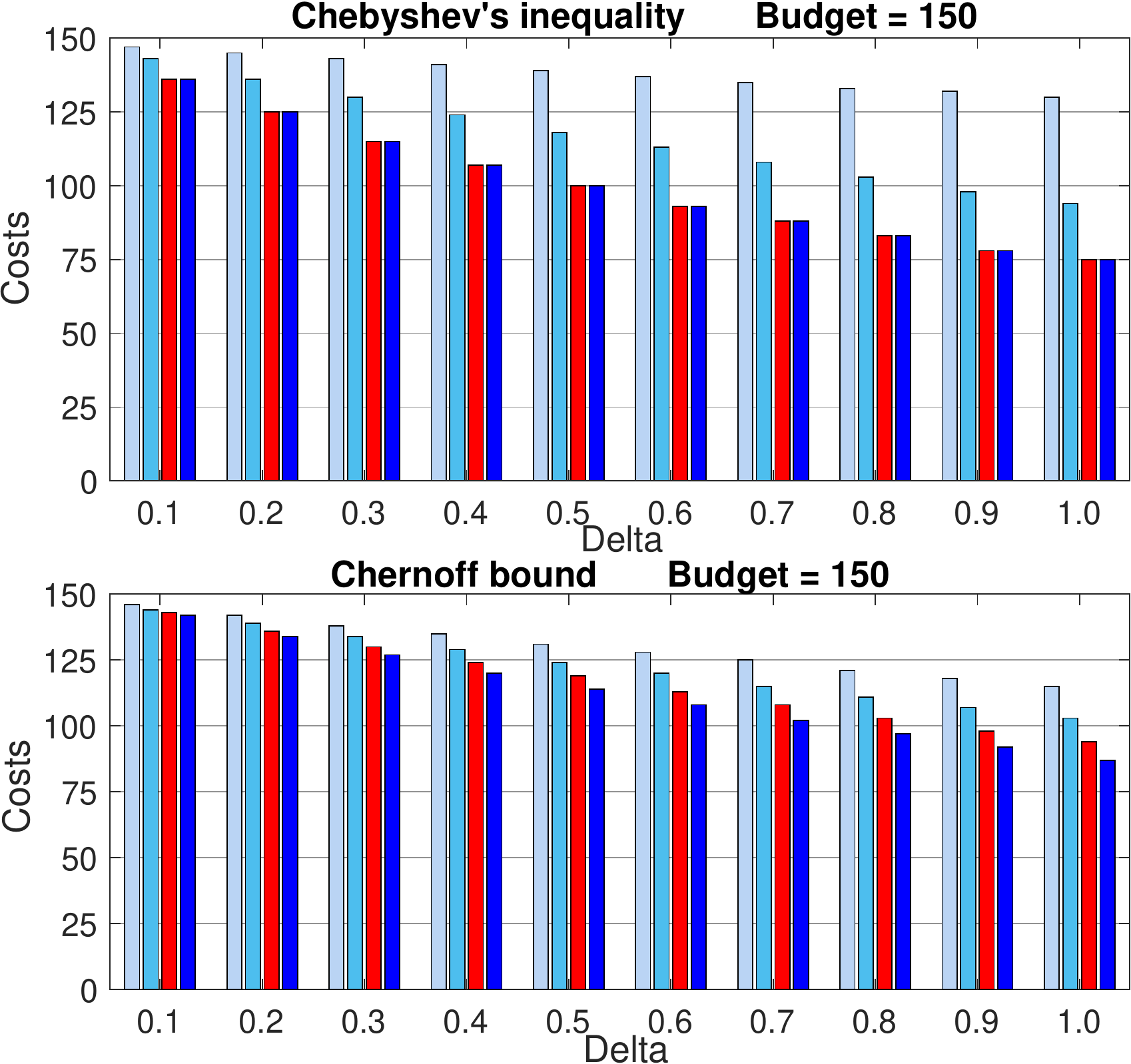}\\
\includegraphics[width=0.4\textwidth]{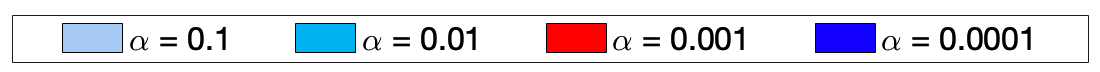}
\vspace{-2mm}
\caption{Maximal cost values for budgets $B=20, 50, 100, 150$ (from left to right) using Chebyshev's inequality (top) and Chernoff bound (bottom) for $\alpha$ = $0.1$, $0.01$, $0.001$, $0.0001$ with uniform expected weights set to $1$.}
\label{fig:plot_costs}
\end{figure*}


\begin{figure*}[t]
\centering
\includegraphics[width=0.22\textwidth]{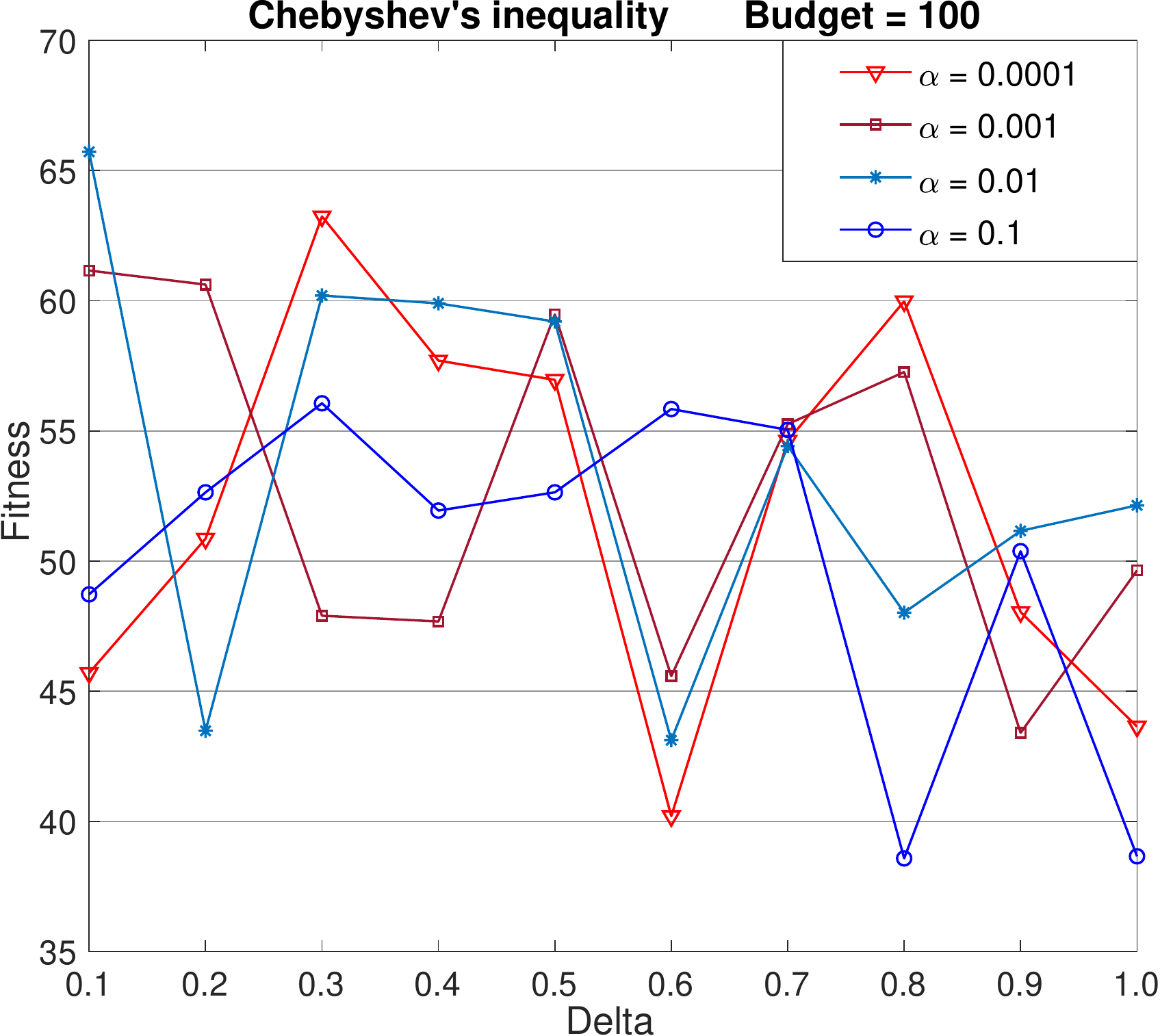}
\includegraphics[width=0.215\textwidth]{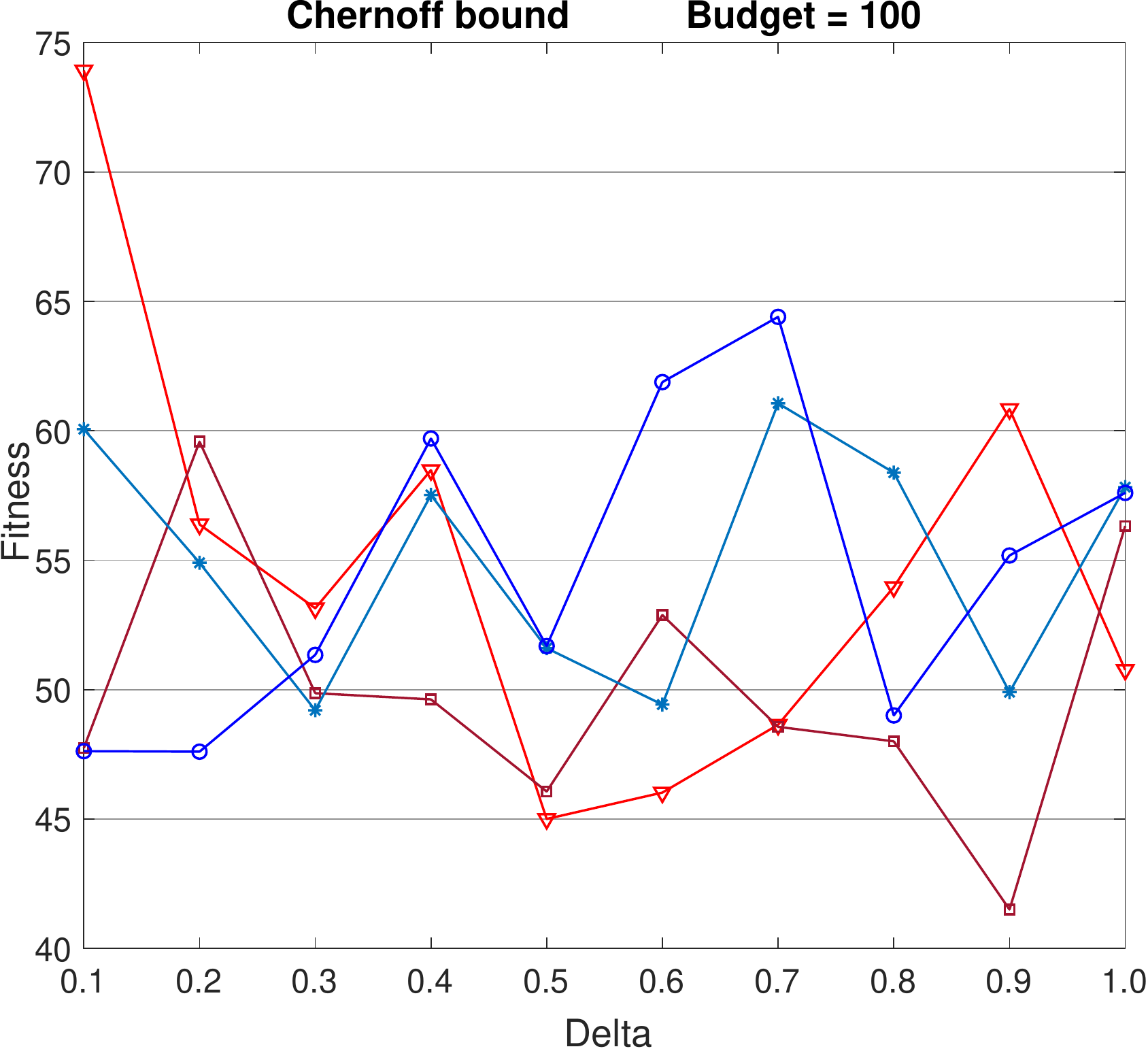}
\includegraphics[width=0.22\textwidth]{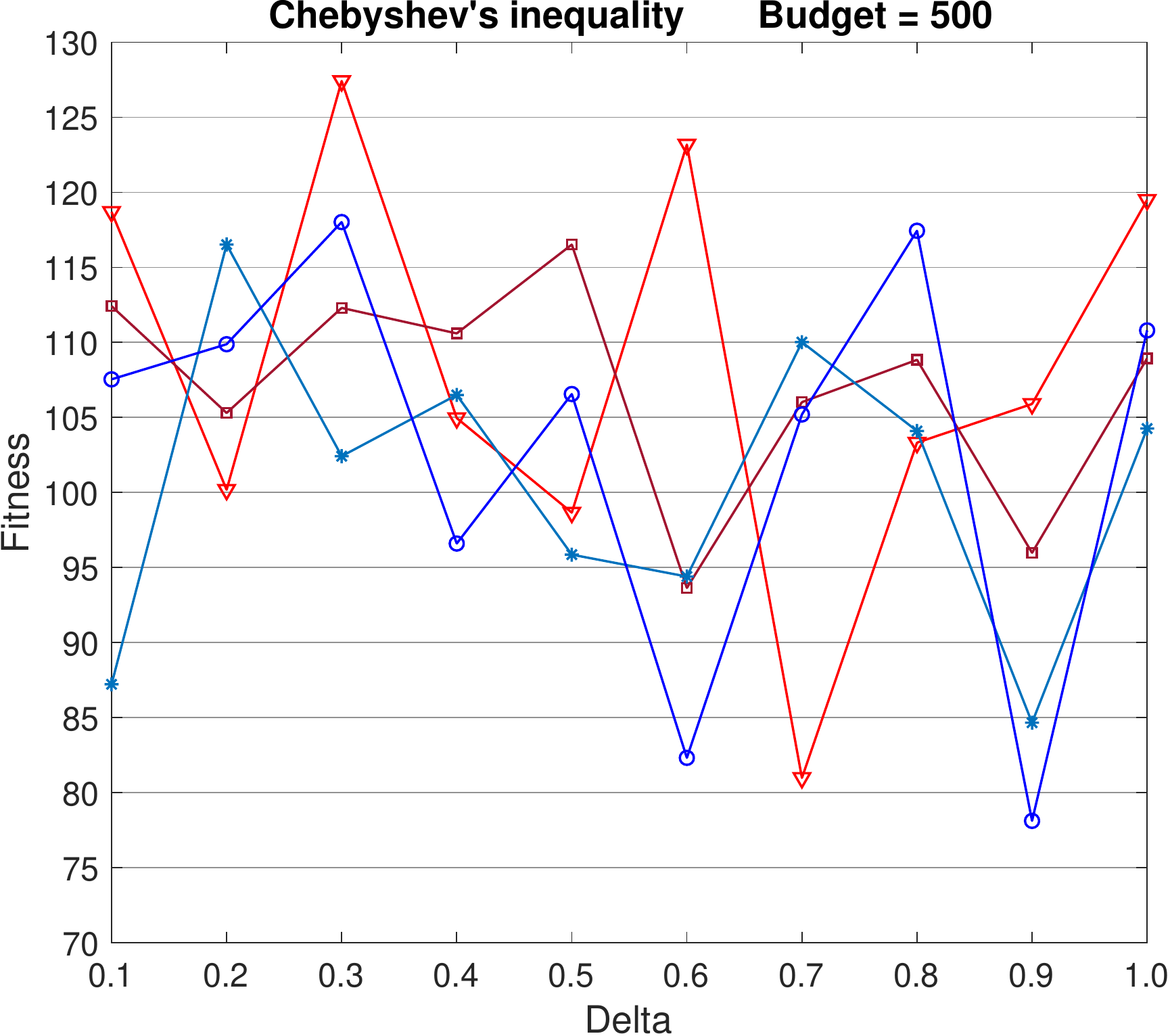}
\includegraphics[width=0.22\textwidth]{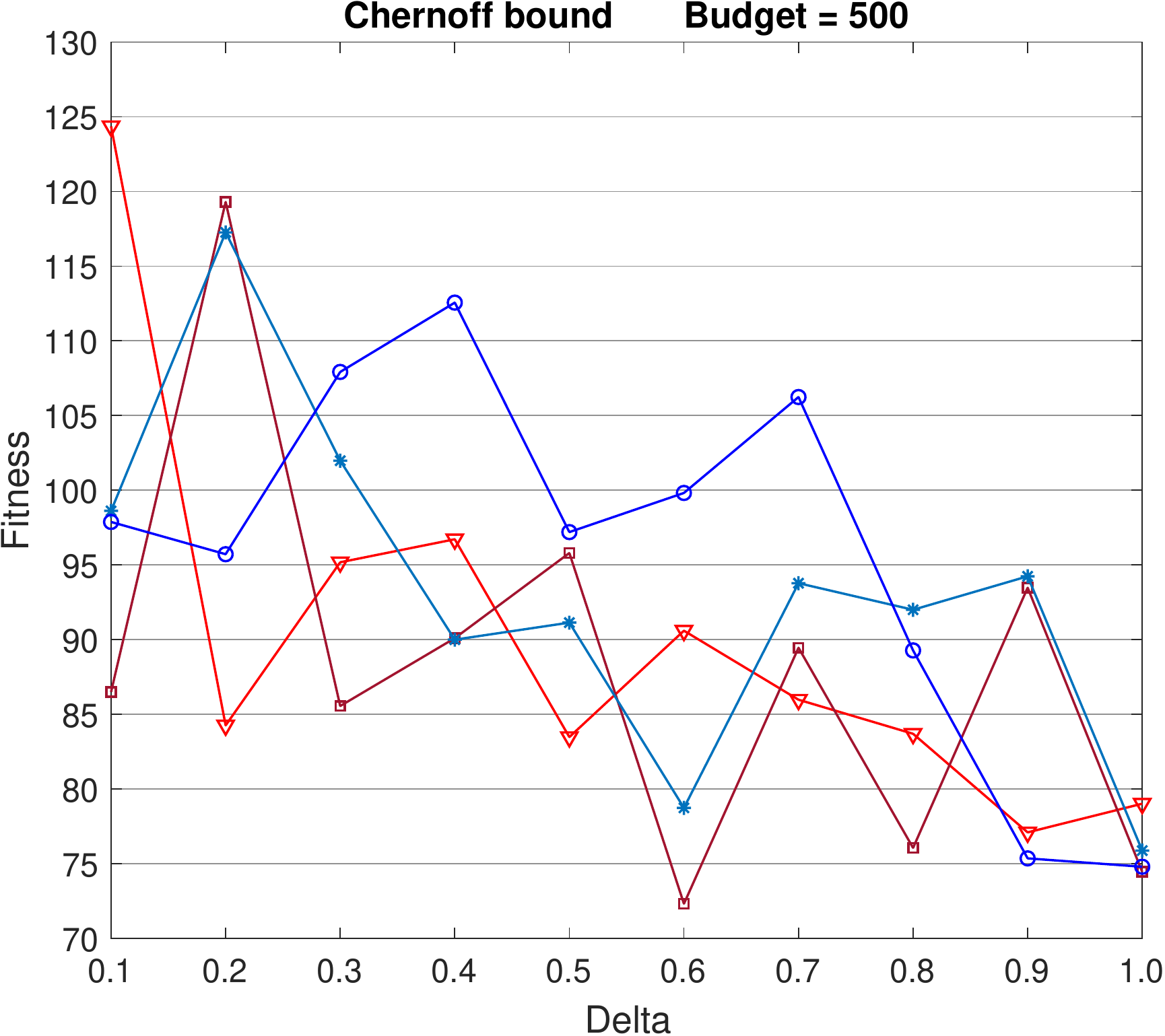}
\vspace{-2mm}
\caption{Function values for budgets $B = 100$ (left) and $B=500$ (right) using Chebyshev's inequality and Chernoff bound for $\alpha$ = $0.1$, $0.01$, $0.001$, $0.0001$ with degree dependent random weights.}
\label{fig:plot2}
\end{figure*}

%

\begin{figure}[t]
\centering
\includegraphics[width=0.22\textwidth]{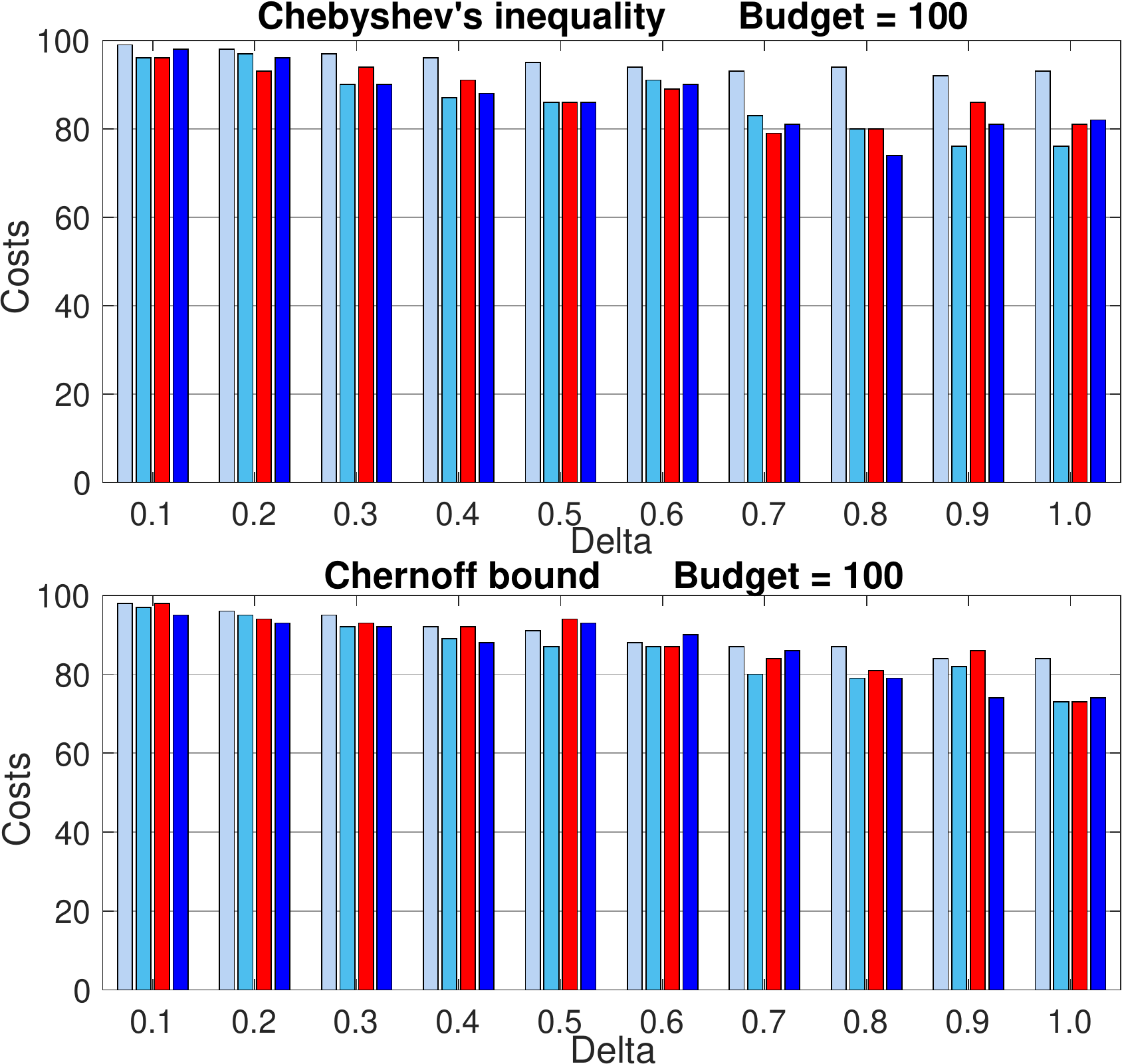}
\includegraphics[width=0.22\textwidth]{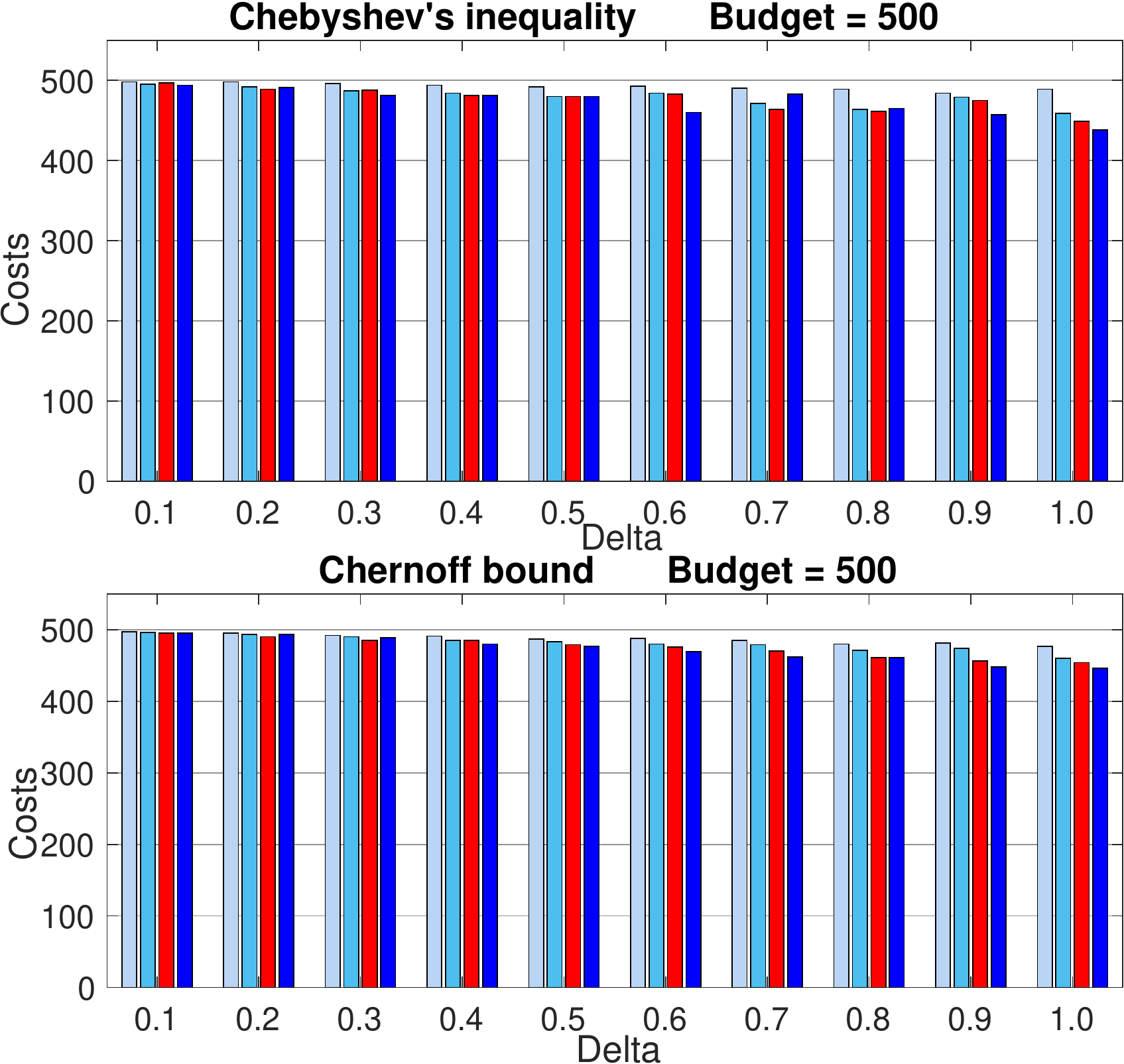}
\includegraphics[width=0.4\textwidth]{alpha.png}
\vspace{-2mm}
\caption{Maximal cost values for budget $B=100$ (left) and $B=500$ using Chebyshev's inequality (top) and Chernoff bound (bottom) for $\alpha$ = $0.1$, $0.01$, $0.001$, $0.0001$ with degree dependent random weights.}
\label{fig:plot3}
\end{figure}

\section{Experimental Investigations}
\label{sec:experiments}

We examine our greedy algorithms on the submodular influence maximization problem~\cite{DBLP:conf/aaai/ZhangV16,DBLP:conf/ijcai/QianSYT17,DBLP:conf/kdd/LeskovecKGFVG07}.
We study the impact of Chebyshev's inequality and Chernoff bounds for the chance-constrained optimization problem in the context of various stochastic parameter settings.

\subsection{The Influence Maximization Problem}

The influence maximization problem (IM) involves finding a set of most influential users in a social network. IM aims to maximize the spread of influence through a social network, which is the graph of interactions within a group of users~\cite{DBLP:conf/kdd/KempeKT03}. A selection of users encounters a cost and the problem of influence maximization has been studied subject to a deterministic constraint that limits the cost of the user selection.

Formally, the problem investigated in our experiments is defined as follows.
Given a directed graph $G=(V,E)$ where each node represents a user, and each edge $(u,v) \in E$ has assigned an edge probability $p_{u,v}$ that user $u$ influences user $v$, the aim of the IM problem is to find a subset $X \subseteq V$ such that the expected number of activated nodes $E[I(X)]$ of $X$ is maximized. 
Given a cost function $c:V\rightarrow \R^+$ and a budget $B\ge 0$, the corresponding submodular optimization problem under chance constraints is given as
$$\argmax_{X\subseteq V} E[I(X)] \text{ s.t. } \Pr[c(X)> B]\leq \alpha.$$ 
A more detailed description of the influence maximization problem on social networks is available at~\cite{DBLP:conf/kdd/LeskovecKGFVG07,DBLP:journals/toc/KempeKT15,DBLP:conf/aaai/ZhangV16,DBLP:conf/ijcai/QianSYT17}. Note though, that these works all study the case in which the cost of adding a user is deterministic.

We consider two types of cost constraints matching the settings investigated in the theoretical analysis.
In the first type, the expected weight of all nodes is $1$, i.e. $a(v)=1$, for all $v \in V$ whereas in the second setting the expected cost of a node $v$ is given by $a(v) = deg(v)+1$. 
Here $deg(v)$ denotes the degree of $v$ in the given graph $G$.
We examine GA for the first type and GGA for the second type of instances.
The chance constraint for both settings requires that the probability of exceeding the given bound $B$ is at most $\alpha$. We investigate different reliability thresholds given by $\alpha$ together with a range of $\delta$ values which determine the amount of uncertainty.
\ignore{
that compute the sum of the weights of all chosen nodes. One constraint takes into account different random weights of chosen nodes of the graph whereas the second constraint sums weights that have all the expected weights $1$. We define for both cost functions, the constraint is met if the cost is at most $B$. 
}
To evaluate the influence maximization problem in the context of surrogate chance constraints, we employ a graph obtained from social news data, with simple settings collected from
a real-world data set obtained from the social news aggregator Digg. The Digg dataset~\cite{DBLP:journals/epjds/HoggL12} contains stories submitted to the platform over a period of a month, and identification (IDs) of users who voted on the popular stories.
The data consist of two tables that describe friendship links among users and the anonymized user votes on news stories~\cite{DBLP:journals/epjds/HoggL12,DBLP:conf/aaai/RossiA15}.
We use the preprocessed data with $3523$ nodes and $90244$ edges, and estimated edge probabilities from the user votes based on the method in~\cite{barbieri2012topic}. 


\subsection{Uniform Chance Constraints}

We consider the results for the greedy algorithm based on Chebyshev's inequality and the greedy algorithms based on Chernoff bounds subject to the uniform distribution with all the expected weights $1$. Figure~\ref{fig:plot} shows the results of influence spread maximization for the GA based on Chebyshev's inequality (first row) and the GA based on Chernoff bounds (second row) for budgets $B$ = $20$, $50$, $100$, $150$. For the experimental investigations of our GA, we consider all combinations of $\alpha = 0.1, 0.01, 0.001, 0.0001$, and $\delta = 0.1, 0.2, 0.3, 0.4, 0.5, 0.6, 0.7, 0.8, 0.9, 1.0$. 


We compare the results in terms of the fitness achieved at each $\alpha$ and $\delta$ level for budgets $B =20, 50, 100, 150$. For the deterministic setting, where $B$ elements can be used, the obtained function values for $B = 20, 50, 100, 150$ are $161.42, 214.78, 287.56$, and $345.84$, respectively. For the chance-constrained problem, Figure~\ref{fig:plot} shows that the GA using Chernoff chance estimates has a better performance/higher fitness values than the GA using Chebyshev's inequality among the considered $\delta$ values in most of the cases. It should be noted that higher budget values have a larger influence on the separation between the two surrogate constraints. In these cases, the Chernoff-based GA performs consistently better, and this is the result of the good asymptotic behavior of the Chernoff bound~\cite{chernoff1952,Mitzenmacher2005}. 

We also compare the results in terms of the expected cost of the solution obtained by the GA based on Chebyshev's inequality and Chernoff bound for budgets $B$ = $20, 50, 100, 150$ as the tail inequalities limit the maximum cost. Note that cost in this case is the same as the number of items picked by GA. The results are shown in Figure~\ref{fig:plot_costs}. The GA using the Chernoff bound is able to include more elements.
The results show that the GA using the Chernoff bound allows for solutions with a larger number of elements if the budget $B$ is high and $\alpha$ is small, for example $B$ = $150$, and for $\alpha$ = $0.001$, $0.0001$. The GA using Chebyshev's inequality has a better performance in the case of high $\alpha$ values, $\alpha$ = $0.1, 0.01$ on the examined problem. 

\subsection{Non-Uniform Chance Constraints}
We now consider the GGA for the setting where the expected weights are not uniform but depend on the degree of the node of the graph. We consider budgets $B=100, 500$ keeping the other parameters the same as in the uniform case.
Figure~\ref{fig:plot2} and \ref{fig:plot3} show the function and cost values obtained
by the GGA based on Chernoff bounds for the combinations of $\alpha$ and $\delta$. 
We do not observe a clear difference between using Chebyshev's inequality or Chernoff bound for the different combinations of $\delta$ and $\alpha$ which is due to a relatively small number of elements present in the solutions.
The function values of the approaches are overall decreasing for each $\alpha$ with increasing value of $\delta$. 
Due to the stochastic behavior of the evaluation, the obtained fitness values show a jagged, irregular course with increasing value of $\delta$.


\section{Conclusion}
Chance constraints play a crucial role when dealing with stochastic components in constrained optimization. We presented a first study on the optimization of submodular problems subject to a given chance constraint.
We have shown that popular greedy algorithms making use of Chernoff bounds and Chebyshev's inequality provide provably good solutions for monotone submodular optimization problems with chance constraints. Our analysis reveals the impact of various stochastic settings on the quality of the solution obtained. Furthermore, we have provided experimental investigations that show the effectiveness of the considered greedy approaches for the submodular influence maximization problem under chance constraints.

\ignore{
\carola{Going forward, we plan to extend our analysis to the case in which the weights do not necessarily show the same dispersion $\delta$, i.e., where $w(s)$ are chosen uniformly in some interval $[a(s)-\delta(s),a(s)+\delta_2(s)]$. It seems likely that for further generalizations to cases with not necessarily uniformly distributed weights, different surrogates are}
}

\section{Acknowledgements}
This work has been supported by the Australian Research Council through grant DP160102401, by the South Australian Government through the Research Consortium "Unlocking Complex Resources through Lean Processing", by the Paris Ile-de-France region, and by a public grant as part of the
Investissement d'avenir project, reference ANR-11-LABX-0056-LMH,
LabEx LMH.
}

\newcommand{\etalchar}[1]{$^{#1}$}
\providecommand{\bysame}{\leavevmode\hbox to3em{\hrulefill}\thinspace}
\providecommand{\MR}{\relax\ifhmode\unskip\space\fi MR }
\providecommand{\MRhref}[2]{%
  \href{http://www.ams.org/mathscinet-getitem?mr=#1}{#2}
}
\providecommand{\href}[2]{#2}

\end{document}